\documentclass[twoside]{article}

\usepackage[accepted]{aistats2024}
\usepackage[round]{natbib}

\usepackage[utf8]{inputenc} % allow utf-8 input
\usepackage[T1]{fontenc}    % use 8-bit T1 fonts
\usepackage{hyperref}       % hyperlinks
\usepackage{url}            % simple URL typesetting
\usepackage{booktabs}       % professional-quality tables
\usepackage{amsfonts}       % blackboard math symbols
\usepackage{nicefrac}       % compact symbols for 1/2, etc.
\usepackage{microtype}      % microtypography
\usepackage{xcolor}         % colors
\usepackage{multicol}

\usepackage{amsmath}
\usepackage{algorithm}
\usepackage{graphicx}
\usepackage{subfigure}
\usepackage{algpseudocode}
\usepackage{amsthm}
\newtheorem{theorem}{Theorem}
\newtheorem{lemma}[theorem]{Lemma}
\newtheorem{claim}[theorem]{Claim}
\usepackage{bbm}
\usepackage{dsfont}
\begin{document}

% If your paper is accepted and the title of your paper is very long,
% the style will print as headings an error message. Use the following
% command to supply a shorter title of your paper so that it can be
% used as headings.
%
%\runningtitle{I use this title instead because the last one was very long}

% If your paper is accepted and the number of authors is large, the
% style will print as headings an error message. Use the following
% command to supply a shorter version of the authors names so that
% they can be used as headings (for example, use only the surnames)
%
%\runningauthor{Surname 1, Surname 2, Surname 3, ...., Surname n}

\twocolumn[

\aistatstitle{Best Arm Identification with Resource Constraints}

\aistatsauthor{ Zitian Li \And Wang Chi Cheung}

\aistatsaddress{ Department of ISEM, \\ National University of Singapore \And  Department of ISEM, \\ National University of Singapore} ]

\begin{abstract}
  Motivated by the cost heterogeneity in experimentation across different alternatives, we study the Best Arm Identification with Resource Constraints (BAIwRC) problem. The agent aims to identify the best arm under resource constraints, where resources are consumed for each arm pull. We make two novel contributions. We design and analyze the Successive Halving with Resource Rationing algorithm (SH-RR). The SH-RR achieves a near-optimal non-asymptotic rate of convergence in terms of the probability of successively identifying an optimal arm. Interestingly, we identify a difference in convergence rates between the cases of deterministic and stochastic resource consumption.
\end{abstract}

\section{Introduction}
\label{Introduction}
Best arm identification (BAI) is a fundamental multi-armed bandit formulation on pure exploration. The over-arching goal is to identify an optimal arm through a sequence of adaptive arm pulls. The efficiency of the underlying strategy is typically quantified by the number of arm pulls. On one hand, the number of arm pulls provides us \emph{statistical insights} into the strategy's performance. On the other hand, the number of arm pulls does not necessarily provide us with the \emph{economic insights} into the total cost of the arm pulls in the scenario of \emph{arm cost heterogeneity}, where the costs of arm pulls differ among different arms. 

Arm cost heterogeneity occurs in a variety of applications. For example, consider a retail firm experimenting two marketing campaigns: (a) advertising on an online platform for a day, (b) providing \$5 vouchers to a selected group of recurring customers. The firm wishes to identify the more profitable one out of campaigns (a, b). The executions of (a, b) lead to different costs. A cost-aware retail firm would desire to control the total cost in experimenting these campaign choices, rather than the number of try-outs. Arm cost heterogeneity also occurs in many other operations research applications. Process design decisions in business domains such as supply chain, service operations, pharmaceutical tests typically involve experimenting a collection of alternatives and identifying the best one in terms of profit, social welfare or any desired metric. Compared to the total number of try-outs, it is more natural to keep the total cost in experimentation in check. What is the relationship between the total arm pulling cost and the probability of identifying the best arm?

We make three contributions to shed light on the above question. Firstly,  we construct the Best Arm Identification with Resource Constraints (BAIwRC) problem model, which features arm cost heterogeneity in a fundamental pure exploration setting. In the BAIwRC, pulling an arm generates a random reward, while consuming resources. The agent, who is endowed with finite amounts of resources, aims to identify an arm of highest mean reward, subject to the resource constraints. In the marketing example, the resource constraint could be the agent's financial budget for experimenting different campaign, and the agent's goal is to identify the most profitable campaign while not exceeding the budget. 

Secondly, we design and analyze the Successive Halving with Resource Rationing (SH-RR) algorithm. SH-RR eliminates sub-optimal arms in phases, and rations an adequate amount of resources to each phase to ensure sufficient exploration in all phases. We derive upper bounds for SH-RR on its $\Pr(\text{fail BAI})$, the probability failing to identify a best arm. Our bounds decay exponentially to zero when the amounts of endowed resources increases. In addition, our bounds depends on the mean rewards and consumptions of the arms. Our bounds match the state-of-the-art when specialized to the fixed budget BAI setting. Crucially, we demonstrate the near-optimality of SH-RR by establishing lower bounds on $\Pr(\text{fail BAI})$ by any strategy. %These lower bounds nearly match our establish upper bounds to SH-RR. 

Thirdly, our results illustrate a fundamental difference between the deterministic and the stochastic consumption settings. Campaign (a) in the marketing example has a deterministic consumption, since the cost for advertisement is determined and fixed (by the ad platform) before (a) is executed. Campaign (b) has a random consumption, since the total cost is proportional to the number of recurring customers who redeem the vouchers, which is random. More precisely, for a BAIwRC instance $Q$, our results imply that $-\log(\Pr(\text{fail BAI}))$ under an optimal strategy is proportional to $\gamma^{\text{det}}(Q)$ or $\gamma^{\text{sto}}(Q)$, in the cases when $Q$ is a deterministic consumption instance or a stochastic consumption instance respectively. The complexity terms $\gamma^{\text{det}}(Q),\gamma^{\text{sto}}(Q)$ are defined in our forthcoming Section \ref{sec:main}, and $\gamma^{\text{det}}, \gamma^{\text{sto}}$ differ in how the resource consumption are encapsulated in their respective settings. 
Finally, our theoretical findings are corroborated with numerical simulations, which demonstrate the empirical competitiveness of SH-RR compared to existing baselines.

% \subsection{Literature Review}
\textbf{Literature Review. } 
The BAI problem has been actively studied in the past decades, prominently under the two settings of fixed confidence and fixed budget. In the fixed confidence setting, the agent aims to minimize the number of arm pulls, while constraining $\Pr(\text{fail BAI})$ to be at most an input confidence parameter. In the fixed budget setting, the agent aims to minimize $\Pr(\text{fail BAI})$, subject to an upper bound on the number of arm pulls. The fixed confidence setting is studied in \citep{EvenMM02,MannorT04,audibert2010best,GabillonGL12,KarninKS13,JamiesonMNB14,KaufmanCG16,GarivierK16}, and surveyed in \citep{JamiesonN14}. The fixed budget setting is studied in \citep{GabillonGL12,KarninKS13,KaufmanCG16,CarpentierL16}. The BAI problem is also studied in the anytime setting \citep{audibert2010best,jun_anytime_2016}, where a BAI strategy is required to recommend an arm after each arm pull. A related objective to BAI is the minimization of simple regret, which is the expected optimality gap of the identified arm, is studied \cite{BubeckMS09,audibert2010best, ZhaoSSJ22}. Despite the volume of studies on pure exploration problems on multi-armed bandits, existing works focus on analyzing the total number of arm pulls. We provide a new perspective by considering the \emph{total cost of arm pulls}. 

BAI problems with constraints have been studied in various works. \cite{WangWJ22} consider a BAI objective where the identified arm must satisfy a safety constraint. \cite{HouTZ2022} consider a BAI objective where the identified arm must have a variance below a pre-specified threshold. Different from these works that impose constraints on the identified arm, we impose constraints on the exploration process. \cite{SuiGBK15,SuiZBY2018} study BAI problems in the Gaussian bandit setting, with the constraints that each sampled arm must lie in a latent safety set. Our BAI formulation is different in that we impose cumulative resource consumption constraints across all the arm pulls, rather than constraints on each individual pull. In addition, \cite{SuiGBK15,SuiZBY2018}  focus on the comparing against the best arm within a certain reachable arm subset, different from our objective of identifying the best arm out of all arms.

Our work is thematically related to the Bandits with Knapsack problem (BwK), where the agent aims to maximize the total reward instead of identifying the best arm under resource constraints. The BwK problem is proposed in \cite{BadanidiyuruKS18}, and an array of different BwK models have been studied \citep{AgrawalD14,AgrawalD16,SankararamanS18}. In the presence of resource constraints, achieving the optimum under the BwK objective does not lead to BAI. For example, in a BwK instance with single resource constraint, it is optimal to pull an arm with the highest mean reward per unit resource consumption, which is generally not an arm with the highest mean reward, when resource consumption amounts differ across arms. 
%----------Literature Review for optimal Arms Identification with Knapsacks-----
\cite{li2023optimal} propose a BAI problem with BwK setting and shares some settings with us, assuming multiple resources with random consumptions, a finite arm set. But their task is to identify the index set ${\cal X}^*$ of all optimal arms in an LP relaxation to a BwK problem. ${\cal X}^*$ depends on both the mean reward and mean consumption of the arms. The difference of the target marks a significant departure from our methodologies and results in the field.
%----------Literature Review for optimal Arms Identification with Knapsacks, ends-----

% Lastly, our work is related to the research on cost-aware Bayesian optimization (BO) 
% %{\color{red} 
% \citep{snoek2012practical, poloczek2017multi, swersky2013multi,lee2020cost,ivkin2021cost}, whereby there is an arm-dependent cost for an arm pull. In this area, an arm might correspond to a hyper-parameter or a combination of hyper-parameters. A general way to determine the pulling sequence is to set up different acquisition functions to guide the selection of sampling points. 

Lastly, our work is related to the research on cost-aware Bayesian optimization (BO). In this area, an arm might correspond to a hyper-parameter or a combination of hyper-parameters. A widely adopted idea in the BO community is to set up different acquisition functions to guide the selection of sampling points. One of the most popular choices is Expected Improvement(EI) \citep{frazier2018tutorial}, without considering the heterogeneous resource consumptions like time or energy. To make EI cost-aware, which is correlated to our setting of a single resource, a common way is to divide it by an approximated cost function $c(x)$ \citep{snoek2012practical, poloczek2017multi, swersky2013multi}, calling it Expected Improvement per unit (EIpu). However, \cite{lee2020cost} shows this division may encourage the algorithm to explore domains with low consumption, leading to a worse performance when the optimal point consumes more resources. Then \cite{lee2020cost} designs the Cost Apportioned BO (CArBo) algorithm, whose acquisition function gradually evolves from EIpu to EI. For better performance, \cite{guinet2020pareto} develops Contextual EI to achieve Pareto Efficiency. \cite{abdolshah2019cost} discusses Pareto Front when there are multiple objective functions. \cite{luong2021adaptive} considers EI and EIpu as two arms in a multi-arm bandits problem, using Thompson Sampling to determine which acquisition function is suitable in each round. These works focus on minimizing the total cost, different from our resource constrained setting. In addition, we allow the resource consumption model to be unknown, random and heterogeneous among arms. In comparison, existing works either assume one unit of resource consumed per unit pulled \citep{jamieson2016non,li2020system,bohdal2022pasha,zappella2021resource}, or assume heterogeneity (and deterministic) resource consumption but with the resource consumption (or a good estimate of it) of each arm known \cite{snoek2012practical,ivkin2021cost,lee2020cost}. Some alternative cost-aware BO require the multi-fidelity or other grey-box assumption \citep{forrester2007multi, kandasamy2017multi, wu2020practical, foumani2023multi, belakaria2023bayesian}, which are not consistent with our settings.

% We provide a detailed review in appendix \ref{sec:detailed_lr}.

\textbf{Notation.} For an integer $K>0$, denote $[K] = \{1, \ldots, K\}$. For $d\in [0, 1]$, we denote $\text{Bern}(d)$ as the Bernoulli distribution with mean $d$.

\section{Model}
\label{sec:model-problem_formulation}
An instance of Best Arm Identification with Resource Constraints (BAIwRC) is specified by the triple $Q = ([K], C ,\nu= \{\nu_k\}_{k\in [K]})$. The set $[K]$ represents the collection of $K$ arms. There are $L$ types of different resources. The quantity $C = (C_\ell)^L_{\ell=1}\in \mathbb{R}_{>0}^L$ is a vector, and $C_\ell$ is the amount of type $\ell$ resource units available to the agent. For each arm $k\in [K]$, $\nu_k$ is the probability distribution on the $(L+1)$-variate outcome $(R_k; D_{1,k}, \ldots, D_{L, k})$, which is received by the agent when s/he pulls arm $k$ once. By pulling arm $k$ once, the agent earns a random amount $R_k$ of reward, and consumes a random amount $D_{\ell, k}$ of the type-$\ell$ resource, for each $\ell\in\{1, \ldots, L\}$. We allow $R_k, D_{1, k}, \ldots, D_{L, k}$ to be arbitrarily correlated. We assume that $R_k$ is a 1-sub-Gaussian random variable, and $D_{\ell, k}\in [0, 1]$ almost surely for every $k\in [K], \ell\in [L]$. 

We denote the mean reward $\mathbb{E}[R_k] = r_k$ for each $k\in [K]$, and denote the mean consumption $\mathbb{E}[D_{\ell, k}] = d_{\ell, k}$ for each $\ell\in [L], k\in [K]$. Similar to existing works on BAI, we assume that there is a unique arm with the highest mean reward, and without loss of generality we assume that $r_1 > r_2 \geq \ldots, \geq r_K$. We call arm 1 the optimal arm. We emphasize that the mean consumption amounts $\{d_{\ell, k}\}^K_{k=1}$ on any resource $\ell$ need not be ordered in the same way as the mean rewards. We assume that $d_{\ell, k} > 0$ for all $k\in [K], \ell\in [L]$. Crucially, the quantities $r_k, d_{\ell, k}, \nu_k$ for any $k, \ell$ are not known to the agent. 

\textbf{Dynamics. }The agent pulls arms sequentially in time steps $t = 1, 2, \ldots$, according to a non-anticipatory policy $\pi$. We denote the arm pulled at time $t$ as $A(t)\in [K]$, and the corresponding outcome as $O(t) = (R(t); D_1(t), \ldots, D_L(t))\sim \nu_{A(t)}$. A non-anticipatory policy $\pi$ is represented by the sequence $\{\pi_t\}^{\infty}_{t=1}$, where $\pi_t$ is a function that outputs the arm $A(t)$ by inputting the information collected in time $1, \ldots, t-1$. More precisely, we have $A(t) = \pi_t(H(t-1))$, where $H(t-1) = \{O(s)\}^{t-1}_{s=1}$. The agent stops pulling arms at the end of time step $\tau$, where $\tau$ is a finite stopping time\footnote{For any $t$, the event $\{\tau = t\}$ is $\sigma(H(t))$-measurable, and $\Pr(\tau = \infty) = 0$} with respect to the filtration $\{\sigma(H(t))\}^\infty_{t=1}$. Upon stopping, the agent identifies arm $\psi\in [K] $ to be the best arm, using the information $H(\tau)$. Altogether, the agent's strategy is represented as $(\pi, \tau, \psi)$. 

\textbf{Objective.} The agent aims to choose a strategy $(\pi, \tau, \psi)$ to maximize $\Pr(\psi=1)$, the probability of BAI, subject to the resource constraint that $\sum^\tau_{t=1} D_\ell(t) \leq C_\ell$ holds for all $\ell\in [L]$ with certainty. We distinguish between two problem model settings, namely the \textbf{stochastic consumption setting} and the \textbf{deterministic consumption setting}. The former is precisely as described above, where we allow $\{D_{\ell, k}\}_{\ell, k}$ to be arbitrary random variables bounded between 0 and 1. The latter is a special case where $\Pr(D_{\ell, k} = d_{\ell, k})=1$ for all $\ell \in [L], k\in [K]$, meaning that all the resource consumption amounts are deterministic. In the special case when $L=1$ and $\Pr(D_{1, k}=1 \text{ for all $k\in [K]$}) = 1$, the deterministic consumption setting specializes to the fixed budget BAI problem. 

We focus on bounding the failure probability $\Pr(\text{fail BAI}) = \Pr(\psi\neq 1)$ in terms of the underlying parameters in $Q$. The forthcoming bounds are in the form of $\exp(-\gamma(Q))$, where $\gamma(Q) > 0$ can be understood as a complexity term that encodes the difficulty of the underlying BAIwRC instance $Q$. 
To illustrate, in the case of $L=1$, we aim to bound $\Pr(\psi\neq 1)$ in terms of $\exp(-C_1/ H)$, where $H > 0 $ depends on the latent mean rewards and resource consumption amounts. In the subsequent sections, we establish upper bounds on $\Pr(\psi\neq 1)$ for our proposed strategy SH-RR, as well as lower bounds on $\Pr(\psi\neq 1)$ for any feasible strategy. We demonstrate that the complexity term $\gamma(Q)$ crucially on if $Q$ has deterministic or stochastic consumption.

\section{The SH-RR Algorithm}
\label{sec:SH-RR_Algorithm}
Our proposed algorithm, dubbed Sequential Halving with Resource Rationing (SH-RR), is displayed in  Algorithm \ref{alg:Sequential-Halving}. SH-RR iterates in phases $q \in \{0, \ldots, \lceil \log_2 K\rceil\}$. Phase $q$ starts with a \emph{surviving arm set} $\tilde{S}^{(q)}\subseteq [K]$. After the arm pulling in phase $q$, a subset of arms in $\tilde{S}^{(q)}$ is eliminated, giving rise to  $\tilde{S}^{(q+1)}$. After the final phase, the surviving arm set  $\tilde{S}^{(\lceil \log_2 K\rceil)}$ is a singleton set, and its only constituent arm is recommended as the best arm. 
\begin{algorithm}[tb]
   \caption{Sequential Halving with Resource Rationing (SH-RR)}
   \label{alg:Sequential-Halving}
    \begin{algorithmic}[1]
        \State {\bfseries Input:} Total budget $C$, arm set $[K]$.
        \State {\bfseries Initialize} $\tilde{S}^{(0)} = [K]$, $t=1$.
        \State {\bfseries Initialize} $\textsf{Ration}^{(0)}_\ell = \frac{C_\ell}{\lceil\log_2 K\rceil}$ for each $\ell\in [L]$.
        \For{$q=0$ {\bfseries to} $\lceil\log_2 K\rceil-1$}
            \State {\bfseries Initialize} $I^{(q)}_\ell = 0$ $\forall\ell\in [L]$, $H^{(q)} =J^{(q)} = \emptyset$.
            % \alglinelabel{alg:SH-RR-while} 
            \While{$I^{(q)}_\ell \leq \textsf{Ration}^{(q)}_\ell-1$ for all $\ell\in [L]$}\label{alg:SH-RR-while}
                % \State Identify arm index $a(t)\in [|\tilde{S}_q|]$, see (\ref{eq:rr_index}).\alglinelabel{alg:rr}
                \State Identify the arm index $a(t)\in \{1, \ldots, |\tilde{S}^{(q)}|\}$ such that $a(t) \equiv t ~\text{mod}~|\tilde{S}^{(q)}|.$\label{alg:rr}
                \State Pull arm $A(t) = k^{(q)}_{a(t)}\in \tilde{S}^{(q)}$. 
                \State Observe the outcome $O(t)\sim \nu_{A(t)}$. 
                \State Update $I^{(q)}_\ell\leftarrow I^{(q)}_\ell + D_\ell(t)$ for each $\ell\in [L]$.
                \State Update $H^{(q)}\leftarrow H^{(q)}\cup \{(A(t),O(t))\}$.
                \State Update $J^{(q)}\leftarrow J^{(q)}\cup\{t\}$.
                \State Update $t\leftarrow t+1$.
           \EndWhile
       % \State \alglinelabel{alg:emp_mean} Use $H^{(q)}$ to compute empirical means $\{\hat{r}^{(q)}_k\}_{k\in \tilde{S}^{(q)}}$, see (\ref{eq:emp_mean}).
       \State Use $\cup^q_{m=0} H^{(m)}$ to compute empirical means $\{\hat{r}^{(q)}_k\}_{k\in \tilde{S}^{(q)}}$, see (\ref{eq:emp_mean}).\label{alg:emp_mean}
       % \State \alglinelabel{alg:elim} Set $\tilde{S}^{(q+1)}$ be the set of top $\lceil |\tilde{S}^{(q)}| / 2\rceil$ arms with highest empirical mean.
       \State Set $\tilde{S}^{(q+1)}$ be the set of top $\lceil |\tilde{S}^{(q)}| / 2\rceil$ arms with highest empirical mean.\label{alg:elim}
       % \State Set $\textsf{Ration}^{(q+1)}_\ell = \frac{C_\ell}{\lceil\log_2 K\rceil} + ( \textsf{Ration}^{(q)}_\ell-I^{(q)}_\ell)$. \alglinelabel{alg:ration}
       \State Set $\textsf{Ration}^{(q+1)}_\ell = \frac{C_\ell}{\lceil\log_2 K\rceil} + ( \textsf{Ration}^{(q)}_\ell-I^{(q)}_\ell)$. \label{alg:ration}
   \EndFor
   \State Output the arm in $\tilde{S}^{(\lceil\log_2 K\rceil)}$.
    \end{algorithmic}
\end{algorithm}
We denote $\tilde{S}^{(q)} = \{k^{(q)}_1, \ldots, k^{(q)}_{|\tilde{S}^{(q)}|}\}$. In each phase $q$, the agent pulls arms in $\tilde{S}^{(q)}$ in a round-robin fashion. At a time step $t$, the agent first identifies (see Line \ref{alg:rr}) the arm index $a(t)\in \{1, \ldots, |\tilde{S}^{(q)}|\}$ and pulls the arm $k^{(q)}_{a(t)}\in \tilde{S}^{(q)}$. The round robin schedule ensures that the arms in $\tilde{S}^{(q)}$ are uniformly explored. SH-RR keeps track of the amount of type-$\ell$ resource consumption via $I^{(q)}_\ell$. The \textbf{while} condition (see Line \ref{alg:SH-RR-while}) ensures that at the end of phase $q$, the total amount $I^{(q)}_\ell$ of type-$\ell$ resource consumption during phase $q$ lies in $ (\textsf{Ration}^{(q)}_{\ell}-1, \textsf{Ration}^{(q)}_{\ell}]$ for each $\ell\in [L]$. The lower bound ensures sufficient exploration on $\tilde{S}^{(q)}$, while the upper bound ensures the feasibility of SH-RR to the resource constraints, as formalized in the following claim:
\begin{claim}\label{claim:feasibility}
With certainty, SH-RR consumes at most $C_\ell$ units of resource $\ell$, for each $\ell\in [L]$. 
\end{claim}

Proof of Claim \ref{claim:feasibility} is in Appendix \ref{app:claim_feasibility}. 
Crucially, SH-RR maintains the observation history $H^{(q)}$ that is used to determined the arms to be eliminated from $\tilde{S}^{(q)}$. %At the end of phase $q$, for every $k\in \tilde{S}^{(q)}$ we have $\sum_{t\in J^{(q)}}\mathds{1}(A(t) = k) \in \{T^{(q)}, T^{(q)}+1\}$ for some random non-negative integer $T^{(q)}$. 
After exiting the \textbf{while} loop, the agent computes (in Line \ref{alg:emp_mean}) the empirical mean
\begin{equation}\label{eq:emp_mean}
\hat{r}^{(q)}_k = \frac{\sum^q_{m=0}\sum_{t\in J^{(m)}} R(t) \cdot \mathds{1}(A(t) = k)}{\max\{\sum^q_{m=0}\sum_{t\in J^{(m)}} \mathds{1}(A(t) = k), 1\}} 
\end{equation}
for each $k\in \tilde{S}^{(q)}$. The surviving arm set $\tilde{S}^{(q+1)}$ in the next phase of phase $q+1$ consists of the $\lceil |\tilde{S}^{(q)}| / 2\rceil$ arms in $\tilde{S}^{(q)}$ with the highest empirical means, see Line \ref{alg:elim}. The amounts of resources rationed for phase $q+1$ is in Line \ref{alg:ration}.

\section{Performance Guarantees of SH-RR}\label{sec:main}
We start with the \textbf{deterministic consumption setting}, and some necessary notation. For each $k\in \{2, \ldots, K\}$, we denote $\Delta_k = r_1 - r_k\in [0, 1]$. We also denote $\Delta_1 = r_1 - r_2 = \Delta_2$. Consequently, we have $\Delta_1 =  \Delta_2 \leq \Delta_3\leq \ldots \leq \Delta_K$. 
For each resource type $\ell\in [L]$, we denote $d_{\ell, (1)},d_{\ell, (2)}, \ldots, d_{\ell, (K)}$ as a permutation of $d_{\ell, 1}, d_{\ell, 2},\ldots, d_{\ell, K}$ such that $d_{\ell, (1)}\geq  d_{\ell, (2)} \geq \ldots \geq d_{\ell, (K)}$. We define 
\begin{align}
H_{2, \ell}^{\text{det}}(Q) &= \max_{k\in \{2, \ldots, K\}} \left\{ \frac{\sum^k_{j=1} d_{\ell, (j)} }{\Delta_k^2} \right\},\label{eq:det_H}
\end{align}
which encodes the difficulty of the instance. When we specialize to the fixed budget BAI problem by setting $L = 1$ and $\Pr(D_{1, k} = 1) = 1$ for all $k\in [K]$, the quantity $H_{2, 1}^{\text{det}}(Q)$ is equal to a quantity $H_2$, which a complexity term defined for the fixed budget BAI setting \citep{audibert2010best,KarninKS13}. Our first main result is an upper bound on $\Pr(\text{fail BAI}) = \Pr(\psi\neq 1)$ for our proposed SH-RR in the deterministic consumption setting.
\begin{theorem}
\label{theorem:upper-bound-of-failure-det}
Consider a BAIwRC instance $Q$ in the deterministic consumption setting. SH-RR (Algorithm \ref{alg:Sequential-Halving}) has BAI failure probability $\Pr(\psi \neq 1)$ at most
\begin{align}\label{eq:upper-bound-of-failure-det}
    \lceil\log_2 K\rceil K \exp\left(-\frac{1}{4 \lceil \log_2 K\rceil}\cdot \gamma^{\text{det}}(Q) \right)
\end{align}
where $\gamma^{\text{det}}(Q) = \min_{\ell\in [L]}\{C_\ell / H^{\text{det}}_{2, \ell}(Q)\}$, and $H^{\text{det}}_{2, \ell}(Q)$ is defined in (\ref{eq:det_H}).
\end{theorem}
Theorem \ref{theorem:upper-bound-of-failure-det} is proved in Appendix \ref{pf:thm_upp_det}. The performance guarantee of SH-RR improves when the complexity term $\gamma^{\text{det}}(Q)$ increases. We provide intuitions in the special case of $L = 1$, so $\ell=1$ always. The upper bound (\ref{eq:upper-bound-of-failure-det}) decreases when $C_1$ increases, since more resource units allows more experimentation, hence a lower failure probability. The upper bound (\ref{eq:upper-bound-of-failure-det}) increases when $H^{\text{det}}_{2, 1}(Q)$ increases. Indeed, when $d_{\ell, (k)}$ increases, the agent consumes more resource units when pulling the arm with the $k$-th highest consumption on resource $\ell$, which leads to less arm pulls under a fixed budget. In addition, when $\Delta_k$ decreases, more arm pulls are needed to distinguish between arms $1,k$, leading to a higher $\Pr(\text{fail BAI})$. Observe that $H^{\text{det}}_{2, 1}(Q)$ involves the mean consumption $d_{1, (1)}, \ldots, d_{1, (K)}$ in a non-increasing order, providing a worst-case hardness measure over all permutations of the arms. One could wonder if the definition of $H_{2,\ell}^{\text{det}}$ can be refined in the non-ordered way, i.e. $\max_{k\in \{2, \ldots, K\}}\{ \sum^k_{j=1} d_{\ell, j} /\Delta_k^2 \}$. Our analysis in appendix \ref{sec:Improvement_of_H2_is_unachievable} shows such a refinement is unachievable.
% $poly(K) \exp\left(-\frac{O(1)}{\log_2 K} \min_{\ell\in [L]}\frac{C_{\ell}}{\max_{k\in \{2, \ldots, K\}} \left\{ \frac{\sum^k_{j=1} d_{\ell, j} }{\Delta_k^2} \right\}}\right)$

The insights above carry over to the case of general $L$. The complexity term $\gamma^{\text{det}}(Q)$ involves a minimum over all resource types $[L]$, meaning that the failure probability depends on the bottleneck resource type(s). Finally, when we specialize to the fixed budget BAI setting, the upper bound (\ref{eq:upper-bound-of-failure-det}) matches (up to a multiplicative absolute constant) the BAI failure probability upper bound of the Successive Halving algorithm \cite{KarninKS13}. 

At first sight, it seems Theorem \ref{theorem:upper-bound-of-failure-det} should hold in the \textbf{stochastic consumption setting}. Indeed, if an arm's pull consumes $\text{Bern}(d)$ units of a resource (Let's assume $L=1$ for the discussion), then $N$ arm pulls consume at most $Nd + 2\sqrt{Nd\log(1/\delta)}$ units with probability $\geq 1-\delta$, for any $\delta\in (0, 1)$. With a large enough $N$, for example when $C_1/d$ is sufficiently large, we expect $Nd \geq 2\sqrt{Nd\log(1/\delta)}$. That is, with probability $\geq 1-\delta$ the realized consumption is at most twice of $Nd$, the consumption with $N$ pulls where each pull consumes $d$ units with certainty instead of $\text{Bern}(d)$. It then transpires that (\ref{eq:upper-bound-of-failure-det}) should hold, modulo a different constant (from $1/4$) in the exponent. 

Despite the intuition, a simulation on two instances $Q^{\text{det}}, Q^{\text{sto}}$ suggests the otherwise. Instances $Q^{\text{det}}, Q^{\text{sto}}$ both have with $K=2, L = 1, C=2$. Instances $Q^{\text{det}}, Q^{\text{sto}}$ share the same Bernoulli rewards with means $r_1 = 0.5, r_2 = 0.4$ and the same mean resource consumption $d_1 = d_2 = d$, where $d$ varies. In $Q^{\text{det}}$, an arm pull consumes $d$ units with certainty, while in $Q^{\text{sto}}$ it consumes Bern($d$) per pull. We plot $\log(\Pr(\psi \neq 1))$ under SH-RR against the varying $d$ in Figure \ref{fig:compare}, while other model parameters are fixed.
\begin{figure}
    \centering
    \begin{subfigure}
      \centering
      \includegraphics[width=.48\linewidth]{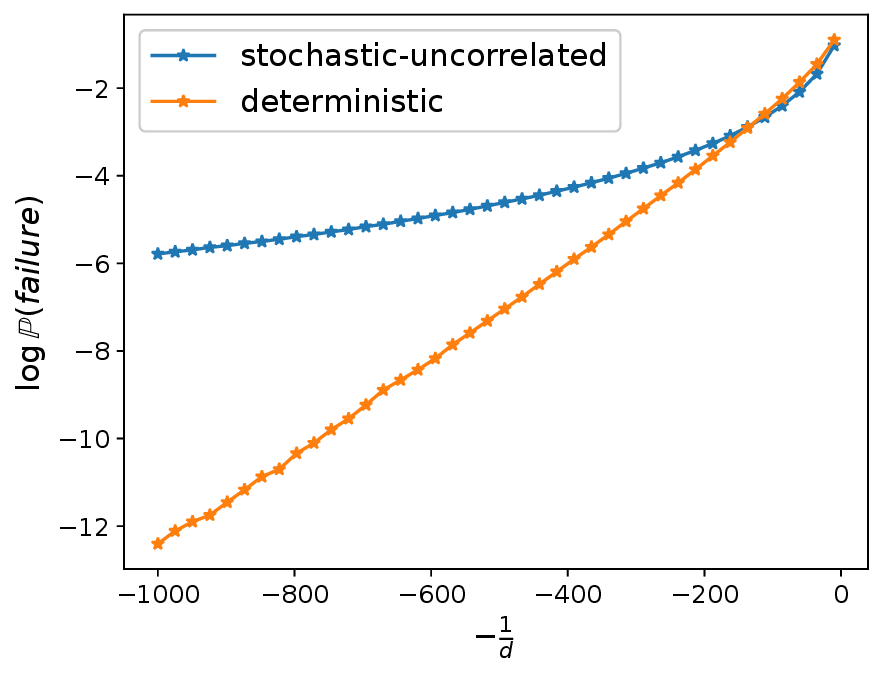}
    \end{subfigure}%
    \begin{subfigure}
      \centering
      \includegraphics[width=.48\linewidth]{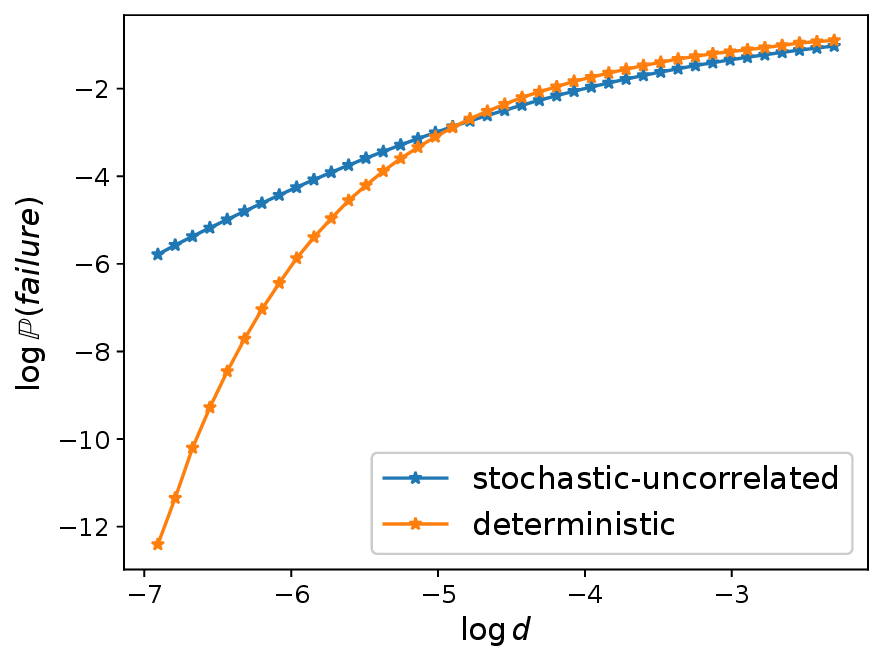}
    \end{subfigure}
    \caption{Convergence rates of $\log(\Pr(\psi \neq 1))$, with $10^7$ repeated trials}
    \label{fig:compare}
\end{figure}
Figure \ref{fig:compare} shows that the $\Pr(\psi \neq 1)$ for $Q^{\text{det}}$ is always less than that for $Q^{\text{sto}}$. In addition, $\Pr(\psi \neq 1)$'s for $Q^{\text{det}}, Q^{\text{sto}}$ diverge when $d$ shrinks, which is in contrary to the previous mentioned intuition. The left panel shows that the plotted $\log(\Pr(\psi \neq 1))$ does not decreases linearly as $1/d$ grows, which implies that the bound in (\ref{eq:upper-bound-of-failure-det}) does not hold for $Q^{\text{sto}}$ when $d$ is sufficiently small.

It turns out the \textbf{stochastic consumption setting} needs a different characterization. For an instance $Q$ with stochastic consumption, define
\begin{align}
H_{2, \ell}^{\text{sto}}(Q) &= \max_{k\in \{2, \ldots, K\}} \left\{ \frac{\sum^k_{j=1}f(d_{\ell, (j)})}{\Delta_k^2} \right\},  \label{eq:sto_H} 
\end{align}
where the function $f:(0, 1]\rightarrow (0, e^2]$ is defined as 
\begin{equation}\label{eq:f}
f(d)=
\begin{cases} 
   e^2 \cdot d  & \text{if } d\in [e^{-2},1], \\
   2(-\log d)^{-1}       & \text{if }  d\in (0, e^{-2}).
\end{cases}
\end{equation}
The function $f$ is continuous and increasing in $(0, 1]$. 
\begin{theorem}
\label{theorem:upper-bound-of-failure}
Consider a BAIwRC instance $Q$ in the stochastic consumption setting. The SH-RR algorithm has BAI failure probability $\Pr(\psi \neq 1)$ at most
\begin{align}\label{eq:upper-bound-of-failure}
7LK (\log_2 K) \exp\left(-\frac{1}{8\lceil \log_2 K\rceil}\cdot\gamma^{\text{sto}}(Q) \right),
\end{align}
where  $\gamma^{\text{sto}}(Q) = \min_{\ell\in [L]}\{C_\ell / H^{\text{sto}}_{2, \ell}(Q)\}$, and $H^{\text{sto}}_{2, \ell}(Q)$ is defined in (\ref{eq:sto_H}).
%{\color{red}Zitian: It should be $\tilde{H}_2 = \max_{2\le k \le K}\frac{\sum_{i=1}^k f(d_{(i)})}{\Delta^2_k}$
\end{theorem}
Theorem \ref{theorem:upper-bound-of-failure} is proved in Appendix \ref{pf:thm_upp_sto}. The upper bound in (\ref{eq:upper-bound-of-failure}) has a similar form to 
% (\ref{eq:upper_main_assert_det})
(\ref{eq:upper-bound-of-failure-det}), except that $\gamma^{\text{det}}(Q)$ is replaced with $\gamma^{\text{sto}}(Q)$. Crucially, the expected consumption $d_{\ell, (j)}$ in $H^{\text{det}}_{2, \ell}(Q)$ is replaced with \emph{effective consumption} $f(d_{\ell, (j)})$ in $H^{\text{sto}}_{2, \ell}(Q)$. For an arm $k\in [K]$, the effective consumption $f(d_{\ell, k})$ encapsulates the magnitude of the random consumption through the mean $d_{\ell, k}$. The non-linearity of $f$ encapsulates the impact of randomness in resource consumption. The function $f(d)$ is increasing in $d$, meaning that a higher mean consumption leads to a higher level of utilization on a resource. Note that $f(d) > d$, and $\lim_{d\rightarrow 0} f(d)/d = \infty$, which bears the following implications. Consider a stochastic consumption instance and a deterministic consumption instance that have the same $\{C_\ell\}_{\ell\in [L]}, \{r_k\}_{k\in [K]}, \{d_{\ell, k}\}_{k\in [K], \ell\in [L]}$. The upper bound (\ref{eq:upper-bound-of-failure}) on $\Pr(\psi\neq 1)$ for the stochastic instance converges at a strictly slower rate to zero than the upper bound  (\ref{eq:upper-bound-of-failure-det}) for the deterministic instance. In addition, when all of $\{d_{\ell, k}\}_{k\in [K], \ell\in [L]}$ tend to zero, the ratio between the two upper bounds grows arbitrarily. These implications are depicted by the diverging curves in Figure \ref{fig:compare}, and the upper bound (\ref{eq:upper-bound-of-failure}) is in fact consistent with the curve for $Q^{\text{sto}}$ in Figure \ref{fig:compare}.

\section{Lower Bounds to $\Pr(\psi\neq 1)$}
\label{sec:main-Lower-Bound}
While the deterioration in the upper bound to $\Pr(\text{fail BAI})$ could appear to be a limitation to SH-RR, we establish lower bound results for BAIwRC, which demonstrate the near optimality of SH-RR. In what follows, we consider instances where arm 1 needs not be optimal, different from our development of SH-RR. Our lower bound results involve constructing $K$ instances $\{Q^{(i)}\}^K_{i=1}$, where %While we adopt the ideas in \citep{CarpentierL16} for constructing the mean rewards, 
the resource consumption models are carefully crafted to (nearly) match the upper bound results of SH-RR.

\textbf{Deterministic consumption setting.} Each instance involves the set $[K]$ of $K$ arms and $L$ types of resources $[L]$.  Let $\{r_k\}^K_{k=1}$ be any sequence such that (\hypertarget{prop:a}{a}) $1/2 = r_1\ge r_2 \ge \cdots \ge r_K \geq 1/4$, and let $\{ \{d_{\ell, (k)}\}^{K}_{k=1} \}_{\ell\in [L]}$ be a fixed but arbitrary collection of $L$ sequences such that (\hypertarget{prop:b}{b}) $d_{\ell,(1)}\ge d_{\ell,(2)}\ge \cdots \ge d_{\ell,(K)}$ for all $\ell\in [L]$, and $d_{\ell, (k)}\in (0, 1]$ for all $\ell\in [L], k\in [K]$. 

In instance $Q^{(i)}$, pulling arm $k\in [K]$ generates a random reward $R_k\sim \text{Bern}(r^{(i)}_k)$, where
$$r^{(i)}_k = \bigg\{ \begin{array}{ll} r_k  & \;\text{if } k\ne i,\\
1 - r_k & \;\text{if } k= i,\end{array}.$$
\textcolor{blue}{Pulling arm $k\in [K]$
consumes 
\begin{equation}\label{eq:consumption_model}
d_{\ell,k} =\Bigg\{\begin{array}{ll}d_{\ell,(2)}&\;\text{if }k=1,\\ d_{\ell,(1)}&\;\text{if }k=2,\\d_{\ell,(k)}&\;\text{if }k\in \{3, \ldots, K\}\end{array}
\end{equation}
units of resource $\ell$ for each $\ell\in [L]$ with certainty.} 

In instance $Q^{(i)}$, arm $i$ is the uniquely optimal arm. All instances $Q^{(1)}, \ldots, Q^{(K)}$ have identical resource consumption model, since the consumption amounts (\ref{eq:consumption_model}) do not depend on the instance index $i$. This ensures that no strategy can extract information about the reward from an arm's consumption. In addition, the consumption amounts in (\ref{eq:consumption_model}) are designed to ensure that (a) instance $Q^{(1)}$ is the hardest among $\{Q^{(i)}\}^K_{i=K}$ in the sense that $H^\text{det}_{2, \ell}(Q^{(1)}) = \max_{i\in [K]}H^\text{det}_{2, \ell}(Q^{(i)})$ for every $\ell\in [L]$, (b) The ordering $d_{\ell,1} \leq d_{\ell,2} \geq d_{\ell,3}\geq \ldots, \geq d_{\ell,K}$ makes $Q^{(1)}$ a hard instance in the sense that it cost the most to distinguish the second best arm (arm 2) from the best arm. More generally, for each resource $\ell$, the consumption amounts are designed such that a sub-optimal arm is more costly to pull when its mean reward is closer to the optimum. Our construction leads to the following lower bound on the performance of any strategy:
\begin{theorem}
    \label{theorem:lower-bound-fixed-consumption-multiple-resource}
    Consider deterministic consumption instances $Q^{(1)}, \ldots, Q^{(K)}$ constructed as above, with $\{r_k\}_{k\in [K]}, \{d_{\ell, (k)}\}_{\ell, k}$ being fixed but arbitrary sequences of parameters that satisfy properties (\hyperlink{prop:a}{a}, \hyperlink{prop:b}{b}) respectively. 
    When $C_1, \ldots, C_L$ are sufficiently large, for any strategy there exists an instance $Q^{(i)}$ (where $i\in [K]$) such that 
    \begin{align*}
        \Pr_{i}(\psi\ne i)\ge &\frac{1}{6}\exp\left(-122\cdot \gamma^{\text{det}}(Q^{(i)})\right),
    \end{align*}
    where $\Pr_{i}(\cdot)$ is the probability measure over the trajectory $\{(A(t), O(t))\}^\tau_{t=1}$ under which the arms are chosen according to the strategy and the outcomes are modeled by $Q^{(i)}$, and $\gamma^{\text{det}}(Q)$ is as defined in Theorem \ref{theorem:upper-bound-of-failure-det}.
\end{theorem}
Theorem \ref{theorem:lower-bound-fixed-consumption-multiple-resource} is proved in Appendix \ref{pf:thm_low_det}. Theorems \ref{theorem:upper-bound-of-failure-det}, \ref{theorem:lower-bound-fixed-consumption-multiple-resource} demonstrate the \textbf{near-optimality of SH-RR}, and the fundamental importance of the quantity $\gamma^{\text{det}}(Q)$ for the BAIwRC problem with deterministic consumption. Indeed, both the BAI failure probability upper bound (of SH-RR) in Theorem \ref{theorem:upper-bound-of-failure-det} and the BAI failure probability lower bound in Theorem \ref{theorem:lower-bound-fixed-consumption-multiple-resource} decay to zero exponentially, with rates linear in $\gamma^{\text{det}}(Q)$. More precisely, the bounds in Theorems \ref{theorem:upper-bound-of-failure-det}, \ref{theorem:lower-bound-fixed-consumption-multiple-resource} imply
\begin{align}
&\underset{\text{strategy}}{\sup} ~\underset{\substack{\text{det inst $Q$:}\\ \gamma^{\text{det}}(Q) \geq \kappa^{\text{det}}}}{\inf}\left\{\frac{- \log\left( \Pr(\text{fail BAI})\right) }{\gamma^{\text{det}}(Q)} \right\} \in\\
&\left[\frac{1}{16 \log_2 K}, 123\right]\label{eq:sup_inf_det}, 
\end{align}
where $\kappa^{\text{det}} = 32 (\log(2K))^{2}$. The supremum is over all feasible strategy, and the infimum is over all instances $Q$ where $\gamma^{\text{det}}(Q) \geq \kappa^{\text{det}}$, i.e. instances with sufficiently large capacities $C_1, \ldots, C_L$. In the special case of fixed-budget BAI, \citep{audibert2010best,CarpentierL16} imply that the right hand side in (\ref{eq:sup_inf_det}) can be $[\frac{1}{8\log_2 K}, \frac{400}{\log K}]$. Pinning down the correct dependence on $\log K$ in (\ref{eq:sup_inf_det}) is an interesting open question.

\textbf{Stochastic consumption setting.} We construct instances $\{Q^{(i)}\}^K_{i=1}$ in a similar way to the case in \text{deterministic consumption setting,} except replacing the consumption model (colored in \textcolor{blue}{blue}) with the following: Pulling arm $k\in [K]$ consumes $D^{(i)}_{\ell, k} \sim \text{Bern}(d_{\ell, k})$ units of resource $\ell$, where $d_{\ell, k}$ is defined in (\ref{eq:consumption_model}). In addition, the reward $R_k$ and $D_{\ell, 1 }, \ldots D_{\ell, K}$ are jointly independent. We have the following lower bound result:
\begin{theorem}
    \label{theorem:lower-bound-sto-consumption-multiple-resource}
   Consider a fixed but arbitrary function $g:[0, +\infty) \rightarrow [0, +\infty)$ that is increasing and $\lim\limits_{d\rightarrow 0^+}\frac{1}{g(d)\log\frac{1}{d}}=+\infty$, $g(0)=0$, as well as any fixed $\{r_k\}_{k=1}^K \subset (0, 1)$, $\frac{1}{2}=r_1 > r_2\geq \cdots \geq r_K=\frac{1}{4}$, any fixed $\{d^0_{\ell, (k)}\}_{k=1, \ell=1}^{K, L} \subset \mathbb{R}$, $d^{0}_{\ell, (1)} \geq d^{0}_{\ell, (2)} \geq \cdots\geq d^{0}_{\ell, (K)}$, and any fixed $i\in\{2,\cdots, K\}$. We can identify $\bar{c} \in (0, 1)$, such that for any $c\in (0, \bar{c})$  and large enough $\{C_{\ell}\}_{\ell=1}^L$, by taking $d_{\ell, (j)}= cd^{0}_{\ell, (j)}, \forall j\in [K], \forall \ell\in[L]$, we can construct corresponding instances $Q^{(j)}$: (1) pulling arm $k\in [K]$ generates a random reward $R_k\sim \mathcal{N}(r^{(j)}_k, 1)$, where $r^{(j)}_k = \bigg\{ \begin{array}{ll} r_k  & \;\text{if } k\ne j,\\1 - r_k & \;\text{if } k= j,\end{array}$, (2) pulling arm $k\in [K]$ consumes $D_{\ell}\sim\text{Bern}(d_{\ell, k})$, $d_{\ell, k} =\Bigg\{\begin{array}{ll}d_{\ell, (2)}&\;\text{if }k=1,\\ d_{\ell,(1)}&\;\text{if }k=2,\\d_{\ell,(k)}&\;\text{if }k\in \{3, \ldots, K\}\end{array}$ for $\ell\in [L]$. The following performance lower bound holds for any strategy:
    % \begin{align*}
    %     & \max_{j\in \{1, i\}}\Pr_{Q^{(j)}, alg}(\psi\neq j)\\
    %     \geq & \exp\left(-2 \min_{\ell\in [L]}\frac{C_{\ell}}{\frac{g(d_{\ell, (1)})}{\Delta_2^2}+\sum_{k=2}^K \frac{g(d_{\ell, (k)})}{\Delta_k^2}}\right)\\
    %     \geq & \exp\left(-2 \min_{\ell\in [L]}\frac{C_{\ell}}{\max\limits_{k\in\{2,3,\cdots, K\}} \frac{\sum_{j=1}^k g(d_{\ell, (j)})}{\Delta_k^2}}\right)
    % \end{align*}
    \begin{align*}
        \max_{j\in \{1, i\}}\Pr_{Q^{(j)}}(\psi\neq j) \geq \exp\left(-2 \tilde{\gamma}^{\text{sto}}(Q^{(j)})\right),
    \end{align*}
    where $\tilde{\gamma}^{\text{sto}}(Q^{(j)}) = \min_{\ell\in [L]}\frac{C_{\ell}}{\tilde{H}_{2, \ell}^{\text{sto}}}$, and $\tilde{H}_{2, \ell}^{\text{sto}}=\max\limits_{k\in\{2,3,\cdots, K\}} \frac{\sum_{j=1}^k g(d_{\ell, (j)})}{\Delta_k^2}$.
\end{theorem}

%--------polished by Chatgpt--------------
In Theorem \ref{theorem:lower-bound-sto-consumption-multiple-resource}, which is proved in Appendix \ref{pf:low_sto}, we establish a lower bound for stochastic consumption in multiple resource scenarios. This theorem, alongside Theorem \ref{theorem:upper-bound-of-failure}, illustrates the near-optimality of the SH-RR approach in stochastic cases, similar to the conclusion in deterministic settings.

In Theorem \ref{theorem:upper-bound-of-failure}, we introduce a novel complexity measure, $H_{2, \ell}^{\text{sto}}(Q)$. This measure is notable for incorporating a term $\frac{1}{\log\frac{1}{d}}$, which exceeds $d$ when $d$ is small. This results in a larger and possibly weaker upper bound compared to the deterministic case and is also different from traditional BAI literature. A pertinent question arises: Is it feasible to refine this term from $\frac{1}{\log\frac{1}{d}}$ to $d$?

Theorem \ref{theorem:lower-bound-sto-consumption-multiple-resource} directly addresses this question, clarifying that such a refinement should not be expected to hold for any given sets $\{r_k\}_{k=1}^K$ and $\{d_{\ell, k}\}_{k=1, \ell=1}^{K, L}$. This result decisively indicates that the term $\frac{1}{\log\frac{1}{d}}$ in the definition of $H_{2, \ell}^{\text{sto}}(Q)$ is irreplaceable with $\frac{1}{(\log\frac{1}{d})^{1+\epsilon}}$ for any $\varepsilon > 0$, and certainly not with $d$ itself. This pivotal finding emphasizes a crucial aspect: stochastic consumption scenarios are inherently more complex than deterministic ones, especially in cases where the mean consumptions are extremely low. 
%--------polished by Chatgpt, end--------------

%-------partially polished by Chatgpt-----
It is crucial to highlight the difference in Theorem \ref{theorem:lower-bound-sto-consumption-multiple-resource}: it asserts $\max_{j\in{1, i}}$ in its conclusion, diverging from the conventional form of $\max_{j\in [K]}$. Additionally, the ratio $\frac{\tilde{H}_{2, \ell}^{\text{sto}}(Q)}{H_{2, \ell}^{\text{sto}}(Q)}$ can approach to 0, given the function $g$ and sufficiently small $\{d_{\ell, (k)}\}_{k=1, \ell=1}^{K, L}$. This gap that might stem from how the pulling times of arm $i$ are approximated. The current derivation, based on the assumption that all resources are allocated to arm $i$, somehow replace the step 2 in appendix \ref{pf:thm_low_det}. This suggests that there is room for improvement in the approximation. A tighter approximation in the exponential term is to be explored.
%-------partially polished by Chatgpt, end-----

\section{Numerical Experiments}
%------------------------by chatgpt 4.0----------------------
We conducted a performance evaluation of the SH-RR method on both synthetic and real-world problem sets. Our evaluation included a comparison of SH-RR against four established baseline strategies: Anytime-LUCB (AT-LUCB) \citep{jun_anytime_2016}, Upper Confidence Bounds (UCB) \citep{BubeckMS09}, Uniform Sampling, and Sequential Halving \citep{KarninKS13} augmented with the doubling trick. Unlike fixed confidence and fixed budget strategies, these baseline methods are \emph{anytime} algorithms, which recommend an arm as the best arm after each arm pull, continuously, until a specified resource constraint is violated. The evaluation was carried out until a resource constraint was breached, at which point the last recommended arm was returned as the identified arm. Fixed confidence and fixed budget strategies were deemed inapplicable to the BAIwRC problem as they necessitate an upper bound on the BAI failure probability and an upper limit on the number of arm pulls in their respective settings. Further details regarding the experimental set-ups are elaborated in appendix \ref{sec:Details-on-the-Numerical-Experiment-Set-ups}.
%------------------------by chatgpt 4.0, ends----------------------

\begin{figure}
    \centering
    \begin{subfigure}
      \centering
      \includegraphics[width=0.5\textwidth]{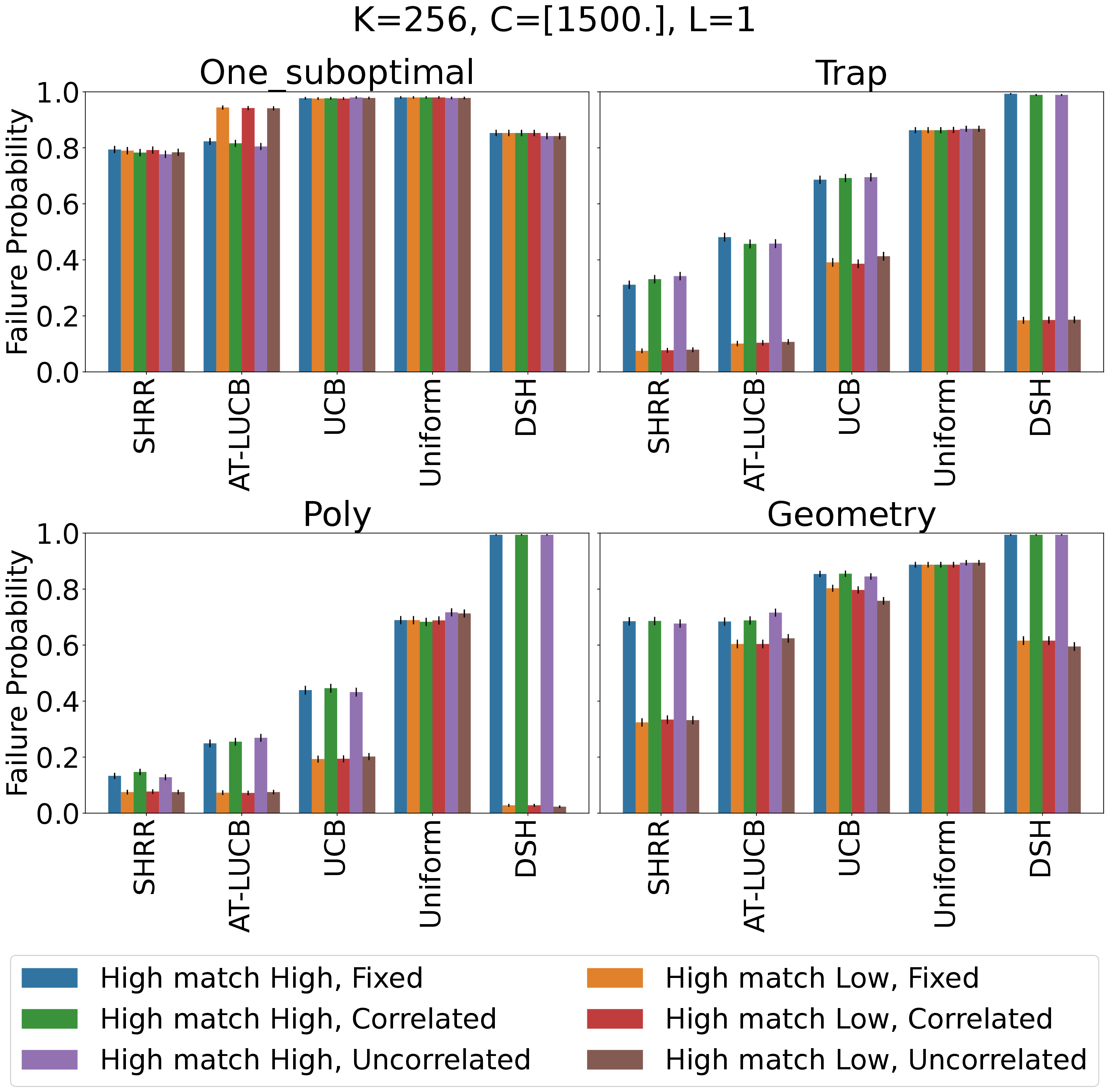}
    \end{subfigure}%
    \begin{subfigure}
      \centering
      \includegraphics[width=0.5\textwidth]{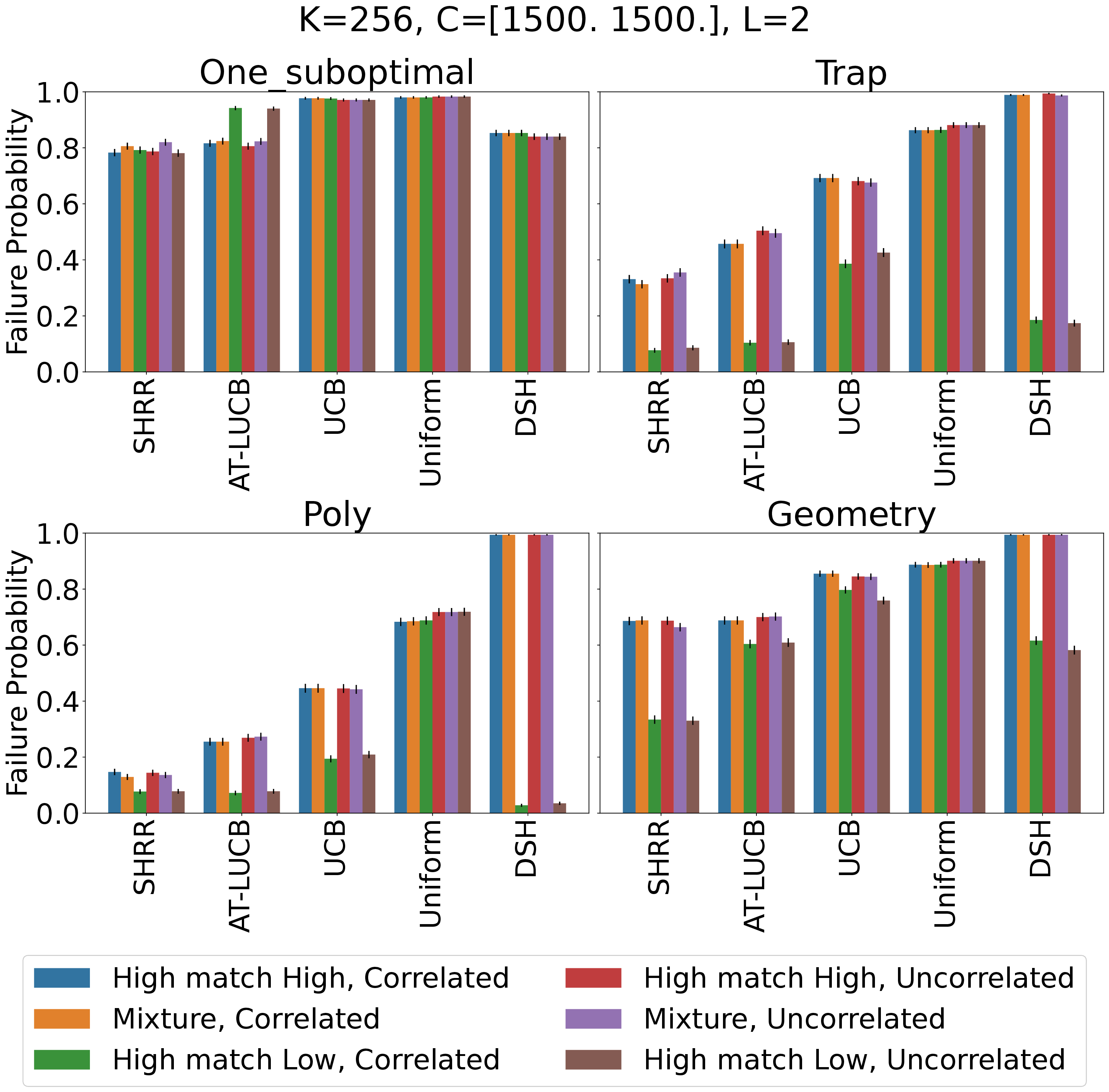}
    \end{subfigure}
    \caption{Comparison of SH-RR and anytime baselines in different setups}
    \label{fig:numeric-result}
\end{figure}

%------------------------Chatgpt 4.0----------------------
\textbf{Synthesis problems.}
We investigated the performance of our algorithm across various synthetic settings, each with distinct reward and consumption dynamics. (1) \emph{High match High} (\textsf{HmH}) where higher mean rewards correspond to higher mean consumption, (2) \emph{High match Low} (\textsf{HmL}) where they correspond to lower mean consumption, and (3) \emph{Mixture} (\textsf{M}) where each arm consumes less of one resource while consuming more of another, applicable when $L=2$. Additionally, resource consumption variability was categorized into deterministic, correlated (random and correlated with rewards), and uncorrelated (random but independent of rewards), with the deterministic setting omitted for $L=2$ due to similar results to $L=1$. Reward variability across arms was explored through four settings: One Group of Sub-optimal, Trap, Polynomial, and Geometric, analogous to \cite{KarninKS13}. The various combinations of these settings are illustrated in Figure \ref{fig:numeric-result}, with more detailed descriptions provided in Appendices \ref{sec:Details-on-the-Numerical-Experiment-Set-ups-single} and \ref{sec:Details-on-the-Numerical-Experiment-Set-ups-multiple}.
%------------------------Chatgpt 4.0 ends----------------------

%------------------------Chatgpt 4.0----------------------
Figure \ref{fig:numeric-result} presents the failure probability of the different strategies in different setups, under {\color{black} $K=256$, $L=1, 2$ with an initial budget of 1500 for each resource}. Each strategy was executed over 1000 independent trials, with the failure probability quantified as {\color{black} $(\text{\# trials that fails BAI}) / 1000$}. Our analysis anticipated a higher difficulty level for the \textsf{HmH} instances, a notion substantiated by the experimental outcomes. The bottom panel illustrates comparable performances on \textsf{M} and \textsf{HmH} setups, suggesting the scarcer resource could predominantly influence the performance. This observed behavior aligns with the utilization of the 
min operator in the definitions of $\gamma^{\text{det}}$ and $\gamma^{\text{sto}}$.
%------------------------Chatgpt 4.0 ends----------------------

Our proposed algorithm SH-RR is competitive compared to these state-of-the-art benchmarks. SH-RR achieves the best performance by a considerable margin for \textsf{HmL}, while still achieve at least a matching performance compared to the baselines for \textsf{HmH}. SH-RR favors arms with relatively high empirical reward, thus when those arms consume less resources, SH-RR can achieve a higher probability of BAI. For confidence bound based algorithms such as AT-LUCB and UCB, resource-consumption-heavy sub-optimal arms are repeatedly pulled, leading to resource wastage and a higher failure probability. These empirical results demonstrate the necessity of SH-RR algorithm for BAIwRC in order to achieve competitive performance in a variety of settings.  

\textbf{Real-world problems.} 
%------------------------Chatgpt 4.0----------------------
We implemented different machine learning models, each adorned with various hyperparameter combinations, as distinct arms. The overarching goal is to employ diverse BAI algorithms to unravel the most efficacious model and hyperparameter ensemble for tackling supervised learning tasks. There is a single constraint on running time for each BAI experiment. To meld simplicity with time-efficiency, we orchestrated implementations of four quintessential, yet straightforward machine learning models: K-Nearest Neighbour, Logistic Regression, Adaboost, and Random Forest. Each model is explored with eight unique hyperparameter configurations. We considered 5 classification tasks, including (1) Classify labels 3 and 8 in part of the MNIST dataset (MNIST 3\&8). (2) Optical recognition of handwritten digits data set (Handwritten). (3) Classify labels -1 and 1 in the MADELON dataset (MADELON). (4) Classify labels -1 and 1 in the Arcene dataset (Arcene). (5) Classify labels on weight conditions in the Obesity dataset (Obesity). See appendix \ref{sec:details-realworld} for details on the set-up.
%------------------------Chatgpt 4.0 ends----------------------

%------------------------Chatgpt 4.0----------------------
We designated the arm with the lowest empirical mean cross-entropy, derived from a combination of machine learning models and hyperparameters, as the best arm. Our BAI experiments were conducted across 100 independent trials. During each arm pull in a BAI experiment round—i.e., selecting a machine learning model with a specific hyperparameter combination—we partitioned the datasets randomly into training and testing subsets, maintaining a testing fraction of 0.3. The training subset was utilized to train the machine learning models, and the cross-entropy computed on the testing subset served as the realized reward. 

The results, showcased in Table \ref{tab: Performance-of-BAI-Real-World} and \ref{tab: Performance-of-BAI-Real-World-cont}, delineate the failure probability in identifying the optimal machine learning model and hyperparameter configuration for each BAI algorithm. Amongst the tested algorithms, SH-RR emerged as the superior performer across all experiments. This superior performance can be attributed to two primary factors: (1) classifiers with lower time consumption, such as KNN and Random Forest, yielded lower cross-entropy, mirroring the \textsf{HmL} setting; and (2) the scant randomness in realized Cross-Entropy ensured that after each half-elimination in SH-RR, the best arm was retained, underscoring the algorithm's efficacy.
%------------------------Chatgpt 4.0 ends----------------------

\begin{table}[h]
\begin{center}
    \caption{Failure Probability of different BAI strategies on Real-life datasets}
    \label{tab: Performance-of-BAI-Real-World}
    \centering
    \begin{tabular}{ cccccc } 
      \toprule
      % BAI Algorithm \textbackslash BAI Task  & MNIST 3\&8 & Handwritten & MADELON
      Algorithm & MNIST 3\&8 & Handwritten
      \\
      \midrule
      SHRR & \textbf{0} & \textbf{0.12}\\
      ATLUCB & 0.21 & 0.23\\  
      UCB & 0.21 & 0.34\\
      Uniform & 0.21  & 0.25\\ 
      DSH & 0.14 &  0.20\\
      \bottomrule
    \end{tabular}
    \end{center}
\end{table}

\begin{table}[h]
\begin{center}
    \caption{Failure Probability of different BAI strategies on Real-life datasets, cont.}
    \label{tab: Performance-of-BAI-Real-World-cont}
    \centering
    \begin{tabular}{ cccccc } 
      \toprule
      % BAI Algorithm \textbackslash BAI Task  & MNIST 3\&8 & Handwritten & MADELON
      Algorithm & Arcene & Obesity  & MADELON
      \\
      \midrule
      % SHRR\_history & \textbf{0} & \textbf{0.12} & \textbf{0.04} & \textbf{0.38} & \textbf{0.31}\\
      SHRR & \textbf{0.38} & \textbf{0.31}  & \textbf{0}\\
      ATLUCB & 0.6 & 0.43  & 0.42\\  
      UCB & 0.71 & 0.43  & 0.30 \\
      Uniform & 0.81 & 0.56   & 0.29\\ 
      DSH & 0.67 & 0.54  & 0.12\\
      \bottomrule
    \end{tabular}
    \end{center}
\end{table}

\section*{Acknowledgement}
The authors are very thankful to Kwang-Sung Jun for his detailed explanations on the AT-LUCB policy and the provision of codes. In addition, the authors would like to thank the reviewing team for the constructive suggestions to strengthen the results. The research is partially funded by a Singapore Ministry of Education AcRF Tier 2 Grant (Project ID: MOE-000238-00, Award Number: MOE-T2EP20121-0012). 
\bibliography{main-reference}

%%%%%%%%%%%%%%%%%%%%%%%%%%%%%%%%%%%%%%%%%%%%%%%%%%%%%%%%%%%%
\section*{Checklist}

 \begin{enumerate}

 \item For all models and algorithms presented, check if you include:
 \begin{enumerate}
   \item A clear description of the mathematical setting, assumptions, algorithm, and/or model. [Yes/No/Not Applicable]
    
   \textbf{Yes. Please check the section \ref{sec:model-problem_formulation} and \ref{sec:SH-RR_Algorithm}.}
   \item An analysis of the properties and complexity (time, space, sample size) of any algorithm. [Yes/No/Not Applicable]

   \textbf{Yes. Please check the section \ref{sec:main}.}
   \item (Optional) Anonymized source code, with specification of all dependencies, including external libraries. [Yes/No/Not Applicable]

   \textbf{Yes. Please check the supplementary file.}
 \end{enumerate}

 \item For any theoretical claim, check if you include:
 \begin{enumerate}
   \item Statements of the full set of assumptions of all theoretical results. [Yes/No/Not Applicable]

   \textbf{Yes. Please check the section \ref{sec:main} and \ref{sec:main-Lower-Bound}.}
   \item Complete proofs of all theoretical results. [Yes/No/Not Applicable]

   \textbf{Yes. Please check the appendix.}
   \item Clear explanations of any assumptions. [Yes/No/Not Applicable] 

   \textbf{Yes.}
 \end{enumerate}

 \item For all figures and tables that present empirical results, check if you include:
 \begin{enumerate}
   \item The code, data, and instructions needed to reproduce the main experimental results (either in the supplemental material or as a URL). [Yes/No/Not Applicable]

   \textbf{Yes. Please check the appendix.}
   \item All the training details (e.g., data splits, hyperparameters, how they were chosen). [Yes/No/Not Applicable]
    \item A clear definition of the specific measure or statistics and error bars (e.g., with respect to the random seed after running experiments multiple times). [Yes/No/Not Applicable]

    \textbf{Yes. Please check the appendix.}
    
    \item A description of the computing infrastructure used. (e.g., type of GPUs, internal cluster, or cloud provider). [Yes/No/Not Applicable]
    
    \textbf{Yes. Please check the appendix.}
 \end{enumerate}

 \item If you are using existing assets (e.g., code, data, models) or curating/releasing new assets, check if you include:
 \begin{enumerate}
   \item Citations of the creator If your work uses existing assets. [Yes/No/Not Applicable]

   \textbf{Yes.}
   \item The license information of the assets, if applicable. [Yes/No/Not Applicable]

   \textbf{Yes.}
   \item New assets either in the supplemental material or as a URL, if applicable. [Yes/No/Not Applicable]

   \textbf{Not Applicable}.
   \item Information about consent from data providers/curators. [Yes/No/Not Applicable]

   \textbf{Not Applicable}.
   \item Discussion of sensible content if applicable, e.g., personally identifiable information or offensive content. [Yes/No/Not Applicable]

   \textbf{Not Applicable}.
 \end{enumerate}

 \item If you used crowdsourcing or conducted research with human subjects, check if you include:
 \begin{enumerate}
   \item The full text of instructions given to participants and screenshots. [Yes/No/Not Applicable]

   \textbf{Not Applicable}.
   \item Descriptions of potential participant risks, with links to Institutional Review Board (IRB) approvals if applicable. [Yes/No/Not Applicable]

   \textbf{Not Applicable}.
   \item The estimated hourly wage paid to participants and the total amount spent on participant compensation. [Yes/No/Not Applicable]
   \textbf{Not Applicable}.
 \end{enumerate}
\end{enumerate}

\newpage
\appendix

% % ---original title----------
% \section{Supplementary Material}
% test
% % ---original title----------
\onecolumn
\section{Auxiliary Results}
\begin{lemma}[Chernoff]
\label{lemma:Chernoff}
Let $\{X_n\}_{n=1}^N$ be i.i.d random variables, where $X_n\in [0,1]$ almost surely, with common mean $\mathbb{E}[X_1]=\mu \in [0,1]$. For any $\mu_+\in (\mu, 1]$, it holds that
\begin{align*}
    \mathbb{P}\left[\frac{1}{N}\sum_{n=1}^NX_n\ge \mu_+\right]\le \exp\left(-N \cdot \text{KL}(\mu_+, \mu)\right),
\end{align*}
where $\text{KL}(p,q)=p\log\frac{p}{q}+(1-p)\log\frac{1-p}{1-q}$ for $p,q\in [0,1]$. In addition, for any $\epsilon\in(0, 2e-1)$, it holds that $$\mathbb{P}\left[\frac{1}{N}\sum_{n=1}^N X_n\ge (1+\epsilon)\mu\right]\le \exp\left(-\frac{N\mu \epsilon^2}{4}\right).$$
\end{lemma}
Lemma \ref{lemma:Chernoff} can be extracted from Exercise 10.3 in \cite{LattimoreS2020}.

\section{Proofs}
\subsection{Proof of Claim \ref{claim:feasibility}}\label{app:claim_feasibility}
\begin{proof}[Proof of Claim \ref{claim:feasibility}]
    The total type-$\ell$ resource consumption is
    \begin{align*}
        &\sum^{\lceil \log_2 K\rceil -1}_{q = 0} I^{(q)}_\ell\nonumber\\
         =  &\sum^{\lceil \log_2 K\rceil -1}_{q = 0} \left[ \frac{C_\ell}{\lceil\log_2 K\rceil} + ( \textsf{Ration}^{(q)}_\ell- \textsf{Ration}^{(q+1)}_\ell  )\right]\\
        = & C_\ell + ( \textsf{Ration}^{(0)}_\ell- \textsf{Ration}^{(\lceil\log_2 K\rceil)}_\ell  ).
    \end{align*}
    We complete the proof by showing that  $\textsf{Ration}^{(0)}_\ell = \frac{C_\ell}{\lceil\log_2 K\rceil}\leq \textsf{Ration}^{(q)}_\ell$ with certainty for every $q$. Indeed, with certainty we have $I^{(q-1)}_\ell \leq \textsf{Ration}^{(q-1)}_\ell$ for every $q \geq 1$. The \textbf{while} loop maintains that $I^{(q - 1)}\leq \textsf{Ration}^{(q-1)}_\ell - 1$, which ensures that $I^{(q - 1)}\leq \textsf{Ration}^{(q-1)}_\ell$ when the \textbf{while} loop ends, and consequently $\frac{C_\ell}{\lceil\log_2 K\rceil}\leq \textsf{Ration}^{(q)}_\ell$ by Line \ref{alg:ration}. Altogether, the claim is shown.
\end{proof}

\subsection{Proof of Theorem \ref{theorem:upper-bound-of-failure-det}}\label{pf:thm_upp_det}

Denote $T^{(q)}$ as the pulling times at the $q^{th}$ phase, $S^{(q)}$ as the surviving set at the $q^{th}$ phase. The size of $S^{(q)}$ is always $\lceil\frac{K}{2^q}\rceil$. Assume the pulling times of each arm in each phase are the same (the difference is at most 1). $T^{(q)}$ is a random number. Conditioned on $S^{(q)}$, $T^{(q)} = \min_{\ell}  \frac{C_{\ell}}{\lceil\log_2 K\rceil \sum_{k\in S^{(q)}} d_{k,\ell}}$. Thus we can assert $\mathbb{P}\left(T^{(q)}\ge \min_{\ell} \frac{C_{\ell}}{\lceil\log_2 K\rceil \sum_{k=1}^{\lceil \frac{K}{2^{q}} \rceil} d_{\ell,(k)}}\right)=1 $. Define $\bar{T}^{(q)}= T^{(1)}+T^{(2)}+\cdots+T^{(q)}$, then we can assert $\mathbb{P}\left(\bar{T}^{(q)}\ge \min_{\ell} \frac{C_{\ell}}{\lceil\log_2 K\rceil \sum_{k=1}^{\lceil \frac{K}{2^{q}} \rceil} d_{\ell,(k)}}\right)=1$.

Denote set $E_q:=\{i: \hat{r}_{i, \bar{T}_q}> \hat{r}_{1, \bar{T}_q}\}$, and define the bad event 
\begin{equation}\label{eq:bad_det}
B^{(q)}=\{ |E_q| \ge \lceil\frac{K}{2^{q+1}}\rceil\}
\end{equation}

We assert that, for any phase $q$
\begin{align}
    \mathbb{P}(B^{(q)})\le K \exp\left(-\min_{\ell} \frac{C_{\ell}}{4\lceil\log_2 K\rceil H_{2, \ell}} \right).\label{eq:upper_main_assert_det}
\end{align}
Once the main assertion (\ref{eq:upper_main_assert_det}) is shown, the Theorem can be proved by a union bound over all phases:
\begin{align*}
\Pr(\psi\neq 1)
\leq & \mathbb{P}(\cup_{q=1}^{\lceil\log_2 K\rceil} B^{(q)})\\
\leq & \sum_{q=1}^{\lceil\log_2 K\rceil} \mathbb{P}(B^{(q)})\\
\leq & \lceil\log_2 K\rceil K \exp\left(-\min_{\ell} \frac{C_{\ell}}{4\lceil\log_2 K\rceil H^{\text{det}}_{2, \ell}} \right).
\end{align*}
In what follows, we establish the main claim (\ref{eq:upper_main_assert_det}). 

For any $q$, we have
\begin{align*}
    &\mathbb{P}(B^{(q)})\\
    \le & \mathbb{P}\left(\exists k \ge \lceil\frac{K}{2^{q+1}}\rceil, \hat{r}_{k, \bar{T}_q} > \hat{r}_{1, \bar{T}_q}\right)\\
    \le & \mathbb{P}\left(\exists k \ge \lceil\frac{K}{2^{q+1}}\rceil, \exists N \geq \min_{\ell} \frac{C_{\ell}}{\lceil\log_2 K\rceil \sum_{k=1}^{\lceil \frac{K}{2^{q}} \rceil} d_{\ell,(k)}}, \hat{r}_{k, N} > \hat{r}_{1, N}\right)\\
    \le & \sum_{k= \lceil\frac{K}{2^{q+1}}\rceil}^K \mathbb{P}\left(\exists N \geq \min_{\ell} \frac{C_{\ell}}{\lceil\log_2 K\rceil \sum_{k=1}^{\lceil \frac{K}{2^{q}} \rceil} d_{\ell,(k)}}, \hat{r}_{k, N} > \hat{r}_{1, N}\right).
\end{align*}
Denote $N_0:= \min_{\ell} \frac{C_{\ell}}{\lceil\log_2 K\rceil \sum_{k=1}^{\lceil \frac{K}{2^{q}} \rceil} d_{\ell,(k)}}$. For any $k$, let $\{R_{k,n},R_{1,n}\}_{n=1}^{\infty}$ be i.i.d samples of rewards under arm k, arm 1. Define $G_n=R_{k,n}-R_{1,n}+\Delta_k$, then $\mathbb{E}G_n=0$. And $\hat{r}_{k, N}=\frac{1}{N} \sum_{n=1}^N R_{k,n}$, $\hat{r}_{1, N}=\frac{1}{N} \sum_{n=1}^N R_{1,n}$. $\hat{r}_{k, N}>\hat{r}_{1, N} \Rightarrow \frac{\sum_{n=1}^N G_n}{N}>\Delta_k$. Take $\lambda = \Delta_k$,
\begin{align*}
    &\mathbb{P}\left(\exists N \geq N_0, \frac{\sum_{n=1}^N G_n}{N}>\Delta_k\right)\\
    =&\mathbb{P}\left(\exists N \geq N_0, \exp(\lambda \sum_{n=1}^N G_n)>\exp(N\lambda \Delta_k)\right)\\
    =&\mathbb{P}\left(\sup_{N\ge N_0} \frac{\exp(\lambda \sum_{n=1}^N G_n)}{\exp(N\lambda \Delta_k)} > 1\right).
\end{align*}
Since $\mathbb{E}\left[\frac{\exp(\lambda \sum_{n=1}^{N+1} G_n)}{\exp((N+1)\lambda \Delta_k)} | G_1, G_2,\cdots, G_{N}\right] \le \frac{\exp(\lambda \sum_{n=1}^{N} G_n)}{\exp(N\lambda \Delta_k)} \frac{\exp(\frac{\lambda^2}{2})}{\exp(\lambda \Delta_k)}\le \frac{\exp(\lambda \sum_{n=1}^{N} G_n)}{\exp(N\lambda \Delta_k)} $, by Doob's optional stopping theorem, see Theorem 3.9 in \cite{LattimoreS2020} for example.
\begin{align*}
    & \mathbb{P}\left(\sup_{N\ge N_0} \frac{\exp(\lambda \sum_{n=1}^N G_n)}{\exp(N\lambda \Delta_k)} > 1\right)\\
    \le & \mathbb{E}\frac{\exp(\lambda \sum_{n=1}^{N_0} G_n)}{\exp(N_0\lambda \Delta_k)}\\
    \le & \frac{\exp(\frac{N_0\lambda^2}{2})}{\exp(N_0\lambda \Delta_k)}\\
    = & \exp(-\frac{N_0\Delta_k^2}{2})\\
    = & \exp\left(-\min_{\ell} \frac{C_{\ell}}{\lceil\log_2 K\rceil \sum_{k=1}^{\lceil \frac{K}{2^{q}} \rceil} d_{\ell,(k)}} \frac{(r_1-r_k)^2}{2}\right)\\
    \le & \exp\left(-\min_{\ell} \frac{C_{\ell}}{\lceil\log_2 K\rceil \sum_{k=1}^{\lceil \frac{K}{2^{q}} \rceil} d_{\ell,(k)}} \frac{(r_1-r_k)^2}{2}\right).
\end{align*}
Thus
\begin{align*}
    \mathbb{P}(B^{(q)})
    \le K \exp\left(-\min_{\ell} \frac{C_{\ell}}{4\lceil\log_2 K\rceil H^{\text{det}}_{2, \ell}} \right).
\end{align*}
Altogether, the Theorem is proved. $\hfill \qed$

\subsection{Proof of Theorem \ref{theorem:upper-bound-of-failure}}\label{pf:thm_upp_sto}
Before the proof, we need a lemma to bound the pulling times of arms.
\begin{lemma}\label{lemma:bound_count}
Let $\{X_n\}_{n=1}^{\infty}$ be i.i.d random variable, $\mathbb{P}\left(X_n\in [0,1]\right)=1$, and denote $\mathbb{E}X_i =d \in [0,1]$. For any positive integer $N$, it holds that
\begin{equation*}
    \mathbb{P}\left(\frac{1}{N}\sum_{n=1}^N X_n > f(d)\right)\le \exp\left(-\frac{N}{3}\right),
\end{equation*}
where the function $f$ is defined in (\ref{eq:f}).
\end{lemma}

We start by defining $\bar{T}^{(q)} = T^{(1)} + \ldots T^{(q)}$ and set $E_q:=\{i: \hat{r}_{i, \bar{T}^{(q)}}> \hat{r}_{1, \bar{T}^{(q)}}\}$. The bad event is 
\begin{equation}
B^{(q)}=\{ |E_q| \ge \lceil\frac{K}{2^{q+1}}\rceil\}.
\end{equation}
We assert that, for any phase $q$, it holds with certainty that
\begin{align}
    &\mathbb{P}\left(B^{(q)}\right)\leq  2LK \exp\left(-\frac{1}{12}\min_{\ell \in[L]} \{ \frac{C_{\ell}}{\lceil\log_2 K\rceil H_{2, \ell}^{\text{sto}}}\}\right).\label{eq:upper_main_assert}
\end{align}
Remark $\{\hat{k}=1\}\supset \cap_{q=1}^{\log_2 K} \neg B^{(q)}$. Once (\ref{eq:upper_main_assert}) is shown, the Theorem can be proved by a union bound over all phases:
\begin{align*}
    &\mathbb{P}(\hat{k}\neq 1)\\
    \le &\mathbb{P}(\cup_{q=1}^{\log_2 K} B^{(q)})\\
    \le &\sum_{q=1}^{\log_2 K} \mathbb{P}(B^{(q)})\\
    \le &2LK (\log_2 K) \exp\left(-\frac{1}{12}\min_{\ell \in[L]} \{ \frac{C_{\ell}}{\lceil\log_2 K\rceil H_{2, \ell}^{\text{sto}}}\}\right)
\end{align*}

In what follows, we establish the main claim (\ref{eq:upper_main_assert}). For our analysis, we define
$\beta^{(q)}_\ell:=\frac{C_\ell}{ \lceil \log_2 K\rceil \cdot \sum_{k=1}^{\lceil \frac{K}{2^q} \rceil} f(d_{\ell, (k)})}\text{, and }\bar{\beta}^{(q)} = \min_{\ell\in [L]}\{\beta^{(q)}_\ell\}.$, then we can split $\mathbb{P}\left(B^{(q)}\right)$ into two parts.
\begin{align*}
    &\mathbb{P}(B^{(q)})\\
    \le & \mathbb{P}\left(\exists k \ge \lceil\frac{K}{2^{q+1}}\rceil, \hat{r}_{k, \tilde{T}^{(q)}} > \hat{r}_{1, \tilde{T}^{(q)}}\right)\\
    \le &\underbrace{ \mathbb{P}\left(\exists k \ge \lceil\frac{K}{2^{q+1}}\rceil, \hat{r}_{k, \bar{T}^{(q)}} > \hat{r}_{1, \bar{T}^{(q)}}, \bar{T}^{(q)} \ge \bar{\beta}^{(q)} \right) }_{(\P)}\\
    &\qquad \qquad \qquad \qquad+ \underbrace{ \mathbb{P}\left(\bar{T}^{(q)} < \bar{\beta}^{(q)}\right) }_{(\ddagger)}
\end{align*}
To facilitate our discussions, we denote $\{\tilde{D}^{(q)}_{\ell, k}(n)\}^\infty_{n=1}$ as i.i.d. samples of the random consumption of resource $\ell$ by pulling arm $k$. For the term $(\ddagger)$, we have
\begin{align}
& \mathbb{P}\left(\bar{T}^{(q)} < \bar{\beta}^{(q)} \mid \tilde{S}^{(q)}  \right) \\
\leq &\mathbb{E}\left[\mathds{1}\left(T^{(q)}< \bar{\beta}^{(q)}
 \right)\mid \tilde{S}^{(q)}  \right]\nonumber\\
=&\mathbb{E}\Biggr[\mathds{1}\Biggr(\sum_{n=1}^{\lceil\bar{\beta}^{(q)}\rceil}\sum_{k\in \tilde{S}^{(q)}}\tilde{D}^{(q)}_{\ell, k}(n)>\textsf{Ration}_\ell^{(q)}\text{ for some $\ell\in[L]$}\Biggr)\mid \tilde{S}^{(q)} \Biggr]\nonumber\\
\leq &\mathbb{E}\Biggr[\mathds{1}\Biggr(\sum_{n=1}^{\lceil\bar{\beta}^{(q)}\rceil}\sum_{k\in \tilde{S}^{(q)} }\tilde{D}^{(q)}_{\ell, k}(n)> \frac{C_\ell}{\lceil \log_2 K\rceil}\text{ for some $\ell\in[L]$} \Biggr)\mid \tilde{S}^{(q)}\Biggr]\label{eq:by_ration}\\
\le & \sum_{\ell\in [L]}\sum_{k\in \tilde{S}^{(q)} }\mathbb{E}\Biggr[\mathds{1}\left(\sum_{n=1}^{\lceil \bar{\beta}^{(q)}\rceil} \tilde{D}^{(q)}_{k}(n)>\frac{C_\ell f(d_{\ell, k})}{\lceil \log_2 K\rceil \cdot \sum_{k'\in \tilde{S}^{(q)} }f(d_{\ell, k'})}\right)\mid \tilde{S}^{(q)} \Biggr]\label{eq:pigeon}\\
\le & \sum_{\ell\in [L]}\sum_{k\in \tilde{S}^{(q)}}\mathbb{E}\Biggr[\mathds{1}\left(\sum_{n=1}^{\lceil \beta^{(q)}_\ell\rceil }\tilde{D}^{(q)}_{k}(n)>\beta^{(q)}_\ell f(d_{\ell, k})\right)\mid \tilde{S}^{(q)}  \Biggr]\label{eq:def_beta}\\
\le & \sum_{\ell\in [L]}\sum_{k\in\tilde{S}^{(q)}} \exp\left(-\frac{\beta^{(q)}_\ell }{3}\right)\label{eq:by_lemma_4_1}\\
\leq & L \cdot |\tilde{S}^{(q)}|\cdot \exp\left(-\frac{\bar{\beta}^{(q)}}{3}\right). \label{eq:for_ddagger}
\end{align}
Step (\ref{eq:by_ration}) is by the invariance $\textsf{Ration}_\ell^{(q)} \geq \frac{C_\ell}{\lceil \log_2 K\rceil}$ maintained by the \textbf{while} loop of SH-RR. Step (\ref{eq:pigeon}) is by the pigeonhole principle. Step (\ref{eq:def_beta}) is by the definition of $\beta^{(q)}_\ell$. Step (\ref{eq:by_lemma_4_1}) is by applying Lemma \ref{lemma:bound_count}. 

Next, we analyze the term $(\P)$, we denote $$(\dagger)^{(q)}_k=\mathbb{P}\left(\hat{r}_{1, \bar{T}^{(q)}}^{(q)}<\hat{r}_{k, \bar{T}^{(q)}}^{(q)},\bar{T}^{(q)}>\bar{\beta}^{(q)}\right).$$ 
We remark that $(\P) \leq  \sum_{k= \lceil\frac{K}{2^{q+1}}\rceil}^K (\dagger)^{(q)}.$ Let $\{R^{(q)}_{k}(n)\}^\infty_{n=1},\{R^{(q)}_{1,n}\}^\infty_{n=1}$ be i.i.d. samples of the rewards under arm $k$ and arm 1 respectively. For each $n$, we define $W(n)=R^{(q)}_{k}(n)-R^{(q)}_{1}(n)+\Delta_k$, where we recall that $\Delta_k = r_1 - r_k$. Clearly, $\mathbb{E}[W(n)]=0$, and $W(n)$ are i.i.d. 1-sub-Gaussian. For any $\lambda> 0$, we have
\begin{align}
    &(\dagger)^{(q)}_k \nonumber\\
    \leq &\mathbb{P}\left(\exists N \geq \bar{\beta}^{(q)}, \frac{\sum_{n=1}^N G_n}{N}>\Delta_k\right)\nonumber\\
    % \leq &\mathbb{E}\left[\mathds{1}\left(\frac{1}{N}\sum_{n=1}^N W(n) > \Delta_k \text{ for some $N> \bar{\beta}^{(q)}$}\right)\mid G^{(q)}  \right]\nonumber\\
    =&\mathbb{P}\left(\sup_{N\ge \bar{\beta}^{(q)}} \frac{\exp(\lambda \sum_{n=1}^N G_n)}{\exp(N\lambda \Delta_k)} > 1\right) \nonumber\\
    \le& \mathbb{E}\left[\frac{\exp\left(\lambda\sum_{n=1}^{\lceil \bar{\beta}^{(q)}\rceil } W(n)\right)}{\exp\left(\lambda \lceil \bar{\beta}^{(q)}\rceil \Delta_k\right)}  \right]\label{eq:doob}\\
    \leq &\frac{\exp\left(\frac{\lambda^2 \lceil \bar{\beta}^{(q)}\rceil}{2}\right)}{\exp\left(\lambda \lceil \bar{\beta}^{(q)}\rceil \Delta_k\right)}\label{eq:rew-subgau}.
\end{align}
Step (\ref{eq:doob}) is by the maximal inequality for (super)-martingale, which is a Corollary of the Doob's optional stopping Theorem, see Theorem 3.9 in \cite{LattimoreS2020} for example. Step (\ref{eq:rew-subgau}) is by the fact that $G(n)$ is 1-sub-Gaussian. Finally, applying $\lambda = \Delta_k$, (\ref{eq:rew-subgau}) leads us to
$
(\dagger)^{(q)}_k \leq \exp(- \bar{\beta}^{(q)}\Delta_k^2 / 2), 
$, meaning
\begin{align}
(\P) &\leq K \exp\left(-\min_{\ell\in [L]}\{\beta^{(q)}_\ell\} \frac{(r_1-r_{ \lceil\frac{K}{2^{q+1}}\rceil })^2}{2}\right)\label{eq:final_P}.
\end{align}
Step (\ref{eq:final_P}) is by the assumption that $\Delta_k$ is not decreasing. Altogether, combining the upper bounds (\ref{eq:by_lemma_4_1},\ref{eq:final_P}) to $(\ddagger) , (\P)$ respectively, leads us to the proof of (\ref{eq:upper_main_assert}).
\begin{align*}
    &\mathbb{P}(B^{(q)})\\
    \le & \sum_{k= \lceil\frac{K}{2^{q+1}}\rceil}^K \mathbb{P}\left(\hat{r}_{k, \bar{T}_q} > \hat{r}_{1, \bar{T}_q}, \mathds{1}(\bar{T}_q \ge \bar{\beta}^{(q)})\right) + L \cdot K \cdot \exp\left(-\frac{1}{3} \min_{\ell\in [L]}\{\beta^{(q)}_\ell\}\right)\\
    \leq & K \exp\left(-\min_{\ell\in [L]}\{\beta^{(q)}_\ell\} \frac{(r_1-r_{\lceil\frac{K}{2^{q+1}}\rceil})^2}{2}\right) + L \cdot K \cdot \exp\left(-\frac{1}{3} \min_{\ell\in [L]}\{\beta^{(q)}_\ell\}\right)\\
    \leq & 2LK \exp\left(-\frac{1}{3}\min_{\ell \in[L]} \{ \frac{C_{\ell}}{4\lceil\log_2 K\rceil H_{2, \ell}^{\text{sto}}}\}\right)\\
    = & 2LK \exp\left(-\frac{1}{12}\min_{\ell \in[L]} \{ \frac{C_{\ell}}{\lceil\log_2 K\rceil H_{2, \ell}^{\text{sto}}}\}\right).
\end{align*}

\subsection{Proof of Lemma \ref{lemma:bound_count}}\label{pf:lemma_bound_count}
The proof involves the consideration of two cases: $d\in(e^{-2},1]$ and $d\in(0, e^{-2}]$. 

\textbf{Case 1: $d\in(e^{-2},1]$.} In this case, we have $f(d)  > 3d$. Consequently,
\begin{align}
& \mathbb{P}\left(\frac{1}{N}\sum_{n=1}^N X_n > f(d)\right)\nonumber\\
    \le &\mathbb{P}\left(\frac{1}{N}\sum_{n=1}^N X_n > 3d\right)\le \exp\left(-\frac{2^2 N d}{4}\right)\label{eq:bound_count_11}\\
    =&\exp(-Nd)\le \exp\left(-\frac{N}{3}\right).\label{eq:bound_count_12}
\end{align}
Step (\ref{eq:bound_count_11}) is by the Chernoff inequality (see Lemma \ref{lemma:Chernoff}), and step (\ref{eq:bound_count_12}) is by the case assumption that $d\geq e^{-2}$. Altogether, \textbf{Case 1} is shown.

\textbf{Case 2: $d\in (0, e^{-2})$.} In this case, note that we still have $f(d) = 2(\log(1/d))^{-1} \geq 2d > d$. We assert that $\text{KL}(f(d), d) \geq 1/2$. Given the assertion, applying Lemma \ref{lemma:Chernoff} gives
\begin{align*}
& \mathbb{P}\left(\frac{1}{N}\sum_{n=1}^N X_n > f(d)\right)\\
   \le &\exp\left(-N \cdot \text{KL}(f(d), d)\right)\\
    \le &\exp\left(-\frac{N}{2}\right)\le \exp\left(-\frac{N}{3}\right),
\end{align*}
which establishes the desired inequality. In the remaining, we show the assertion, which is equivalent to the assertion 
\begin{equation}\label{eq:bound_count_case_2}
\left(\frac{1}{2}\log\frac{1}{d} \right)\cdot \text{KL}(f(d),d)\le \frac{1}{4}\log\frac{1}{d}. 
\end{equation}
To demonstrate (\ref{eq:bound_count_case_2}), we start with the left hand side of (\ref{eq:bound_count_case_2}):
\begin{align}
&\left(\frac{1}{2}\log\frac{1}{d}\right)\cdot  \text{KL}\left(\frac{2}{\log \frac{1}{d}},d\right)\nonumber\\
=&\log\left(\frac{2}{d\log \frac{1}{d}}\right) + \frac{1}{2}\log\frac{1}{d}\left(1-\frac{2}{\log\frac{1}{d}}\right)\log \frac{1-\frac{2}{\log\frac{1}{d}}}{1-d}\nonumber\\
=& \log 2 +\log \frac{1}{d} - \log\log\frac{1}{d}\nonumber\\
&- \underbrace{\left(\frac{1}{2}\log \frac{1}{d}-1\right)\cdot \log\left(1+\frac{\frac{2}{\log\frac{1}{d}}-d}{1-\frac{2}{\log\frac{1}{d}}}\right)}_{(\dagger)}.\label{eq:bound_count_step_0}
\end{align}
We argue that $(\dagger) \leq 1$. Indeed, 
\begin{align}
(\dagger) \le & \left(\frac{1}{2}\log\frac{1}{d}-1\right) \cdot \frac{1}{1-\frac{2}{\log\frac{1}{d}}}\cdot \left(\frac{2}{\log\frac{1}{d}}-d\right)\label{eq:bound_count_step_1}\\
=&\left(\frac{1}{2}\log\frac{1}{d}\right) \cdot \left(\frac{2}{\log\frac{1}{d}}-d\right)\nonumber\\
=&1-\frac{d}{2}\log\frac{1}{d}\le 1. \label{eq:bound_count_step_2} 
\end{align}
Step (\ref{eq:bound_count_step_1}) is by the fact that $\log (1+x)\le x$ for all $ x>-1$. Step (\ref{eq:bound_count_step_2}) is by the case assumption that $d\in (0, e^{-2})$. Next, we apply the bound $(\dagger)\leq 1$ to (\ref{eq:bound_count_step_0}), which yields
\begin{align}
&\left[\frac{1}{2}\log\frac{1}{d}\right]\cdot  \text{KL}\left(\frac{2}{\log \frac{1}{d}},d\right)\nonumber\\
\ge & \log 2 +\log \frac{1}{d} - \log\log\frac{1}{d}-1\nonumber\\
\ge & \log \frac{1}{d} - \log\log\frac{1}{d} - 0.5\nonumber\\
\ge & \frac{1}{4}\log\frac{1}{d}.\label{eq:bound_count_step_3}
\end{align}
Step (\ref{eq:bound_count_step_3}) follows from the fact that $\frac{1}{4}\log\frac{1}{d}\ge0.5$ and $\frac{1}{2}\log\frac{1}{d}\ge\log\log\frac{1}{d}$ hold for any $d\in(0, e^{-2})$. Altogether, \textbf{Case 2} is shown and the Lemma is proved.  $\hfill \qed$

\subsection{Proof of Theorem \ref{theorem:lower-bound-fixed-consumption-multiple-resource}}\label{pf:thm_low_det}
To facilitate our discussion, we denote $\mathbb{E}_i[\cdot]$ as the expectation operator corresponding to the probability measure $\Pr_i$. 
Theorem \ref{theorem:lower-bound-fixed-consumption-multiple-resource} is proved in the following two steps. 

\textbf{Step 1.} We show that, under the assumption $\Pr_1(\psi\neq 1) <1/2$, for every $i\in \{2, \ldots, K\}$ it holds that
\begin{equation}\label{eq:det_first}
   \Pr_i(\psi\neq i) \geq \frac{1}{6}\exp\left(-60t_i\left(\frac{1}{2}-r_i\right)^2 -2\sqrt{T\log(12KT)}\right),
\end{equation}
where 
\begin{equation}\label{eq:tiTi_det}
t_i = \mathbb{E}_1[T_i], \quad T_i = \sum^\tau_{t=1} \mathbf{1}(A(t) = i)
\end{equation}
is the number of times pulling arm $i$, and 
\begin{equation}\label{eq:T_det}
T = \min_{\ell\in [L]}\left\{\lfloor\frac{C_{\ell}}{d_{{\ell},(K)}}\rfloor\right\}
\end{equation}
is an upper bound to the number of arm pulls by any policy that satisfies the resource constraints with certainty. Note that if the assumption $\Pr_1(\psi\neq 1) <1/2$ is violated, the conclusion in Theorem \ref{theorem:lower-bound-fixed-consumption-multiple-resource} immediately holds for $Q^{(1)}$. 

\textbf{Step 2.} We show that there exists $i\in \{2, \ldots K\}$ such that 
\begin{align}
t_i (1/2 - r_i)^2 & \leq \text{min}_{\ell\in [L]} \left\{\frac{2 C_\ell}{H^\text{det}_{\ell, 2}(Q_1)}\right\}\label{eq:step_2_det_main}\\
&  \leq \text{min}_{\ell\in [L]} \left\{\frac{2 C_\ell}{H^\text{det}_{\ell, 2}(Q_i)}\right\}.\nonumber
\end{align}
This step crucially hinges on the how the consumption model is set in (\ref{eq:consumption_model}). Finally, Theorem \ref{theorem:lower-bound-fixed-consumption-multiple-resource} follows by taking $C_1, \ldots, C_L$
so large that
$$
\text{min}_{i\in [K],\ell\in [L]} \left\{\frac{2 C_\ell}{H^\text{det}_{\ell, 2}(Q_i)}\right\} \geq \sqrt{T\log(12KT)}.
$$
Such $C_1, \ldots, C_L$ exist. For example, we can take $C_1 = \ldots = C_L = C$, then the left hand side of the above condition grows linearly with $C$, while the right hand side only grows linearly with $\sqrt{C\log C}$. Altogether, the Theorem is shown, and it remains to establish \textbf{Steps 1, 2}.
%  we have $r^{i}_i > r^{(i)}_1 \geq \ldots\geq r^{(i)}_{i-1} \geq r^{(i)}_{i+1}\geq \ldots \geq r^{(i)}_K$. 

\textbf{Establishing on Step 1.} To establish (\ref{eq:det_first}), we follow the approach in \citep{CarpentierL16} and consider the event  $$\mathcal{E}_i=\{\psi=1\}\cap\{T_i\le 6t_i\}\cap \{\xi\}$$
for $i\in [K]$. The quantities $T_i, t_i$ are as defined in (\ref{eq:tiTi_det}), and $\xi$ is an event concerning an empirical estimate on a certain KL divergence term. To define $\xi$, it requires some set up. Denote $\nu^{(i)}_k$ as the outcome distribution of arm $k$ in instance $Q^{(i)}$ (recall that the outcome consists of the reward $R_k$ and the consumption $D_{1, k}, \ldots D_{\ell, k}$, where $R_k$ has different distributions under different $Q^{(i)}$, while the distribution of $D_{1, k}, \ldots D_{\ell, k}$ is invariant across $Q^{(1)}, \ldots, Q^{(K)}$). Define 
\begin{align}
\text{KL}_i& =\text{KL}(\nu^{(i)}_i, \nu^{(1)}_i)\nonumber\\
&=\text{KL}(\text{Bern}(r_k), \text{Bern}(1- r_k)) \label{eq:by_model_i_def}\\
&= \text{KL}( \text{Bern}(1- r_k), \text{Bern}(r_k))\\
&=(1-2r_k)\log\left(\frac{1-r_k}{r_k}\right),\nonumber
\end{align}
where (\ref{eq:by_model_i_def}) is by the fact that the outcomes $\nu^{(i)}_i, \nu^{(1)}_i$ are identical in the resource consumption but only different in reward. In addition, for each $i\in [K], t\in [T_i], $ we define
$$\widehat{\text{KL}}_{i,t}=\frac{1}{t}\sum_{s=1}^t(1 - 2\tilde{R}_{i}(s))\log\frac{1-r_k}{r_k},$$
where $\tilde{R}_i(1), \ldots, \tilde{R}_i(T_i)$ are arm $i$ rewards received during the $T_i$ pulls of arm $i$ in the online dynamics (recall the definition of $T_i$ in (\ref{eq:tiTi_det})). Note that $T_i\leq T$ with certainty. Finally, define confidence radius
$$\textsf{rad}(t) = (\sqrt{2}\log 3)\cdot \sqrt{\frac{\log 12K T}{t}},$$
and the event $\xi$ is defined as 
\begin{equation}\label{eq:xi}
\xi=\left\{\forall i\in [K],t \in [T_i], |\widehat{\text{KL}}_{i,t}|-\text{KL}_i\le 
\textsf{rad}(t)\right\}.
\end{equation}
The event $\xi$ is also considered in \citep{CarpentierL16} when $T$ is part of the problem input instead of a set parameter in (\ref{eq:T_det}), and the following result from \citep{CarpentierL16} still carries over:
\begin{lemma}[Lemma 4 in \cite{CarpentierL16}]
\label{lemma:Lower-bound-of-good-case-Multiple-Resource}
    $\mathbb{P}_{\mathcal{G}^i}(\xi)\ge \frac{5}{6}$ holds for all $i\in [K]$.
\end{lemma}
For every $i\in \{2, \ldots, K\}$, we have
\begin{align}
    % \label{formula:lower-bound-of-E-k}
    & \Pr_i(\psi \neq i)\nonumber\\
    \geq &\mathbb{P}_{i}(\mathcal{E}_i)\nonumber\\
    =&\mathbb{E}_{1}\left(\mathds{1}\{\mathcal{E}_i\}\exp(-T_i\widehat{\text{KL}}_{i,T_i})\right)\label{eq:by_change_of_measure}\\
    \ge&\mathbb{E}_{1}\left(\mathds{1}\{\mathcal{E}_i\}\exp(-T_i KL_i -2\sqrt{T_i\log(12KT)})\right)\label{eq:by_event_xi}\\
    \ge& \exp\left(-6t_i \text{KL}_i -2\sqrt{T \log(12KT)}\right)\cdot\Pr_1(\mathcal{E}_i)\label{eq:by_tiTi}\\
    \ge& \left[ \frac{2}{3} - \Pr_1(\psi \neq 1)\right]\cdot \exp\left(-6t_i \text{KL}_i -2\sqrt{T \log(12KT)}\right)\label{eq:by_bounding}\\
    \ge& \left[ \frac{2}{3} - \Pr_1(\psi \neq 1)\right]\cdot\exp\left(-60t_i \left(\frac{1}{2} - r_i\right)^2 -2\sqrt{T \log(12KT)}\right)\label{eq:by_KL_bound}.
\end{align}
The above calculations establishes \textbf{Step 1}, and we conclude the discussion on \textbf{Step 1} by justifying steps (\ref{eq:by_change_of_measure}-\ref{eq:by_bounding}). 
    
Step (\ref{eq:by_change_of_measure}) is by a change-of-measure identity frequently used in the MAB literature. For example, it is established in equation (6) in \cite{audibert2010best} and Lemma 18 in \cite{KaufmanCG16}. The identity is described as follows. Let $\tau$ be a stopping time with respect to $\{\sigma(H(t))\}^\infty_{t=1}$, where we recall that $H(t)$ is the historical observation up to the end of time step $t$. For any event ${\cal E}\in \sigma(H(\tau))$ and any instance index $i\in \{2, \ldots K\}$, it holds that
\begin{align*}
    \mathbb{P}_{\mathcal{G}^i}(\mathcal{E})=\mathbb{E}_{\mathcal{G}^1}\left[\mathds{1}\{\mathcal{E}\}\exp(-T_i\widehat{\text{KL}}_{i,T_i})\right].
\end{align*}
Consequently, step (\ref{eq:by_change_of_measure}) holds by the fact that the choice of arm $\psi$ only depends on the observed trajectory $\sigma(H(\tau))$, and evidently both $T_i$ and $\xi$ are both $\sigma(H(\tau))$-measurable. Step (\ref{eq:by_event_xi}) is by the event $\xi$. Step (\ref{eq:by_tiTi}) is by the event that $T_i\leq 6t_i$. 

Step (\ref{eq:by_bounding}) is by the following calculations:
\begin{align}
    &\Pr_{1}(\mathcal{E}_i)  = 1-\Pr_{1}(\neg \mathcal{E}_i)\nonumber\\
    \ge&1-\mathbb{P}_{\mathcal{G}^1}(\neg\{\psi=1\})-\mathbb{P}_{\mathcal{G}^1}(\neg\{\xi\})-\mathbb{P}_{\mathcal{G}^1}(\neg\{T_k\le 6t_k\})\nonumber\\
    \ge& \frac{2}{3} - \Pr_1(\psi\neq 1).\label{eq:by_events}
\end{align}
Step (\ref{eq:by_events}) follows from Lemma \ref{lemma:KL-Divergence-Chernoff-ln} which shows $\Pr_1(\xi) \geq 5/6$, and the Markov inequality that shows that for any $i\in \{2, \ldots, K\}$:
\begin{align*}
    \Pr_{i}(T_i\ge 6t_i)\le \frac{1}{6}.
\end{align*}
Finally, step (\ref{eq:by_KL_bound}) is by the fact that $r_i\in [\frac{1}{4}, \frac{1}{2}]$ for all $i\in [K]$, leading to $0\le \text{KL}_i\le 10(1-r_i)^2$ for all $i\in [K]$.

\textbf{Establishing Step 2.} To proceed with \textbf{Step 2}, we first define a complexity term $H^{\text{det}}_{1, \ell}(Q)$, which is similar to $H^{\text{det}}_{2, \ell}(Q)$ but the former aids our analysis. 
For a deterministic consumption instance $Q$ (whose arms are not necessarily ordered as $r_1\geq r_2\geq \ldots, r_K$), we denote $\{r_{(k)}\}^K_{k=1}$ as a permutation of $\{r_k\}^K_{k=1}$ such that $r_{(1)} > r_{(2)}\geq   \ldots \geq  r_{(K)}$. For example, when $Q= Q^{(i)}$, we can have $r_{(1)} = r^{(i)}_i$, $r_{(j+1)} = r^{(i)}_j$ for $j\in \{1,\ldots, i-1\}$, and $r_{(j)} = r^{(i)}_j$ for $j\in \{i+1,\ldots, K\}$. Similarly, denote $\{d_{\ell,(k)}\}^K_{k=1}$ as a permutation of $\{d_{\ell,k}\}^K_{k=1}$ such that $d_{\ell,(1)}\geq d_{\ell,(2)}\geq \ldots \geq d_{\ell,(K)}$. 
Define $\Delta_{(1)} = \Delta_{(2)} = r_{(1)} - r_{(2)}$, and define $\Delta_{(k)} = r_{(1)} - r_{(k)}$ for $k\in \{3, \ldots, K\}$. Now, we are ready to define 
\begin{equation}\label{eq:H_1_det}
H^{\text{det}}_{\ell, 1}(Q) = \sum^K_{k=1} \frac{d_{\ell, (k)}}{\Delta_{(k)}^{2}}.
\end{equation}
%{\color{red}ZITIAN: I think it should be $H^{\text{det}}_{\ell, 1}(Q)$ and $H^{\text{det}}_{\ell, 2}(Q)$ instead of $H^{\text{det}}_{1, \ell}(Q)$ and $H^{\text{det}}_{2, \ell}(Q)$}

In the special case of $L = 1$ and $d_{1, k} = 1$ for all $k \in [K]$, the quantity $H^{\text{det}}_{\ell, 1}(Q)$ is equal to the complexity term $H_1$ defined for BAI in the fixed confidence setting \citep{audibert2010best} (the term $H_1$ is relabeled as $H$ in subsequent research works \cite{KarninKS13, CarpentierL16}). Observe that for any deterministic consumption instance $Q$, we always have 
\begin{equation}\label{eq:observe_H0}
H^{\text{det}}_{\ell, 2}(Q) \leq H^{\text{det}}_{\ell, 1}(Q).
\end{equation}
\iffalse
$$\leq \log(2K)\cdot H^{\text{det}}_{2, \ell}(Q). 
$$
The first inequality is direct from the definitions of $H^{\text{det}}_{2, \ell}(Q),  H^{\text{det}}_{1, \ell}(Q)$, while the second follows from:
\begin{align}
     H^{\text{det}}_{1, \ell}(Q) &\leq \frac{d_{\ell, 1}}{\Delta^2_{(1)}} + \sum^K_{k=2} H^{\text{det}}_{2, \ell}(Q) \cdot \frac{d_{\ell, (k)}}{\sum^k_{j=1} d_{\ell, (j)}}\nonumber\\
     &\leq H^{\text{det}}_{2, \ell}(Q) + \sum^K_{k=2} H^{\text{det}}_{2, \ell}(Q) \cdot \frac{1}{k} \label{eq:by_ordering_det}\\
     &\leq H^{\text{det}}_{2, \ell}(Q) \cdot \log(2K).\nonumber
\end{align}
Step (\ref{eq:by_ordering_det}) is by the fact that $\frac{d_{\ell, 1}}{\Delta^2_{(1)}} \leq H^{\text{det}}_{2, \ell}(Q)$, and the fact that $d_{\ell, (k)} \leq d_{\ell, (j)}$ for $1\leq j\leq k$. 
\fi
In addition, we observe that for any $i\in \{2, \ldots, K\}$ and any $\ell\in [L]$, it holds that
\begin{align}
H^{\text{det}}_{\ell,1}(Q^{(1)})&\ge H^{\text{det}}_{\ell,1}(Q^{(i)}),\label{eq:H_observe1}\\
H^{\text{det}}_{\ell,2}(Q^{(1)})&\ge H^{\text{det}}_{\ell,2}(Q^{(i)}).\label{eq:H_observe2}
\end{align}

After defining $H^{\text{det}}_{\ell, 1}(Q)$, we are ready to proceed to establishing \textbf{Step 2}. Recall that $ T_i = \sum^\tau_{t=1} \mathbf{1}(A(t) = i)$ is the number of arm pulls on arm $i$. By the requirement of feasibility and the definition of $d_{\ell, k}$ in (\ref{eq:consumption_model}), we know that
\begin{align*}
    T_1d_{\ell,(2)}+ T_2d_{\ell,(1)} + \sum_{k=3}^K T_k d_{\ell,(k)}\le C_{\ell}
\end{align*}
holds for all $\ell\in [L]$. Taking expectation $\mathbb{E}_1$ and recalling the definition $t_i = \mathbb{E}_1[T_i]$ in (\ref{eq:tiTi_det}), we show that 
\begin{align*}
    t_1d_{\ell,(2)}+ t_2d_{\ell,(1)} + \sum_{k=3}^K t_k d_{\ell,(k)}\le C_{\ell}
\end{align*}
holds for all $\ell$. From our definition of $H^\text{det}_{\ell,1}(Q^{(1)})$, for every $\ell\in [L]$ we have
\begin{align*}
    &\frac{d_{\ell, 1}}{H^\text{det}_{\ell,1}(Q^{(1)})(\frac{1}{2}-r_2)^2}+\sum_{k=2}^K \frac{d_{\ell, (k)}}{H^\text{det}_{\ell,1}(Q^{(1)})(\frac{1}{2}-r_k)^2}=1,
\end{align*}
which implies that
\begin{align}
    &\frac{2 C_{\ell} d_{\ell,(1)}}{H^\text{det}_{\ell,1}(Q^{(1)})(\frac{1}{2}-r_2)^2} + \sum_{k=3}^K \frac{C_{\ell} d_{\ell,(k)}}{H^\text{det}_{\ell,1}(Q^{(1)})(\frac{1}{2}-r_k)^2}\nonumber\\
    \ge &t_1d_{\ell,(2)}+t_2d_{\ell,(1)}+t_3d_{\ell,(3)}+\cdots+t_Kd_{\ell,(K)}. \label{eq:crucial_step_2}
\end{align}

holds for any $\ell$. Inequality (\ref{eq:crucial_step_2}) implies that for any $\ell\in [L]$, it is either the case that 
$\frac{2 C_{\ell} \cdot d_{\ell,(1)}}{H^\text{det}_{\ell,1}(Q^{(1)})(\frac{1}{2}-r_2)^2} \geq t_2 d_{\ell, (1)},$
or there exists $k_\ell\in \{3, \ldots, K\}$ such that $\frac{C_{\ell} d_{\ell,(k_\ell)}}{H^\text{det}_{\ell,1}(Q^{(1)})(\frac{1}{2}-r_{k_\ell})^2} \ge t_{k_\ell}d_{\ell,(k_\ell)}$. Collectively, the implication is equivalent to saying that for all $\ell\in [L]$, there exists $k_\ell\in \{ 2, \ldots, K\}$ such that
\begin{align*}
    t_{k_\ell}\left(\frac{1}{2}-r_{k_\ell}\right)^2 \leq \frac{2C_{\ell}}{H^\text{det}_{\ell,1}(Q^{(1)})},
\end{align*}
or more succinctly there exists $i\in \{2, \ldots, K\}$ such that 
$$
 t_{i}\left(\frac{1}{2}-r_{i}\right)^2 \leq \min_{\ell\in [L]}\left\{\frac{2C_{\ell}}{H^\text{det}_{\ell,1}(Q^{(1)})}\right\}.
$$
Finally, \textbf{Step 2} is established by the observations (\ref{eq:observe_H0}, \ref{eq:H_observe1}, \ref{eq:H_observe2}). 

\subsection{Proof of Theorem \ref{theorem:lower-bound-sto-consumption-multiple-resource}}\label{pf:low_sto}
Denote $\tau$ as the total number of arm pulls. Before proving theorem theorem \ref{theorem:lower-bound-sto-consumption-multiple-resource} we need to firstly procide a high probability upper bound to $\tau$, with the lemma 2 in \cite{csiszar1998method}.
\begin{lemma}[\cite{csiszar1998method}]
\label{lemma:KL-Divergence-Chernoff-ln}
% Denote $X \sim Binomial(t', d_1)$ Then we have $\mathbb{P}(T\le t')=\mathbb{P}(X\ge C)$, assume $\frac{C}{t'}>d_1$
Denote $\{D_t\}^\infty_{t=1}$ be i.i.d. random variables distributed as $\text{Bern}(d)$, where $d\in (0, 1)$. Let $C$ be a positive real number. Define random variable $\rho=\min\{T: \sum_{t=1}^T D_t \ge C\}$. For any integer $t'\in (C, C / d)$, it holds that% assume $\frac{C}{t'}>d_1$then
\begin{align*}
    \Pr(\rho \le t')\ge \frac{\exp(-t'\log 2\cdot \text{KL}(C / t', d))}{t'+1},
\end{align*}
where we denote $\text{KL}(p, q) = p\log (p/q) + (1-p)\log((1-p)/(1-q))$.
\end{lemma}
We assert the following lemma.
\begin{lemma}
    \label{lemma:Prob-Bound-of-Stopping-Time}
    Denote $T_i$ as the pulling times of arm $i\in \{3, \cdots, K\}$,i.e $T_i=\sum_{s=1}^{\tau}\mathds{1}(A_s=i)$ and $\{d_{\ell_0, (k)}\}_{k=1}^K \subset (0, 1)^K$. For any $\ell_0 \in [K]$, if $\forall k, g(d_{\ell_0, (k)})<\frac{1}{\log\frac{1}{d_{\ell_0, (i)}}}$, $\frac{(r_1-r_i)^2}{(r_1-r_2)^2 \log\frac{1}{d_{\ell_0, (i)}}} + \sum_{k=3}^K \frac{(r_1-r_i)^2}{(r_1-r_k)^2\log\frac{1}{d_{\ell_0,(i)}}} < 1$ and $\log \frac{1}{1-d_{\ell_0, (i)}} < \frac{1}{2}$ all hold, we have
    \begin{align*}
        \Pr_1\left(T_i > \frac{C_{\ell_0}}{\frac{g(d_{\ell_0,(1)})}{(r_1-r_2)^2} + \sum_{k=3}^K \frac{g(d_{\ell_0,(k)})}{(r_1-r_k)^2}}\frac{1}{(r_1-r_i)^2}\right) \leq 1-\exp\left(-\frac{C_{\ell_0}}{\frac{1}{4\log\frac{1}{d_{\ell_0,(i)}}}}\right) 
    \end{align*}
\end{lemma}
\textbf{Remarks:} We can derive a similar conclusion For the case that $i=2$, with assumptions $\forall k, g(d_{\ell_0, (k)})<\frac{1}{\log\frac{1}{d_{\ell_0, (1)}}}$, $\frac{(r_1-r_i)^2}{(r_1-r_2)^2 \log\frac{1}{d_{\ell_0, (1)}}} + \sum_{k=3}^K \frac{(r_1-r_i)^2}{(r_1-r_k)^2\log\frac{1}{d_{\ell_0,(1)}}} < 1$ and $\log \frac{1}{1-d_{\ell_0, (1)}} < \frac{1}{2}$. The details are omitted here.
\begin{proof}
    For simplicity, denote $\bar{T}_i:=\frac{C_{\ell_0}}{\frac{g(d_{\ell_0,(1)})}{(r_1-r_2)^2} + \sum_{k=3}^K \frac{g(d_{\ell_0,(k)})}{(r_1-r_k)^2}}\frac{1}{(r_1-r_i)^2}$, $h = \left(\frac{1}{(r_1-r_2)^2} + \sum_{k=3}^K \frac{1}{(r_1-r_k)^2}\right) (r_1-r_i)^2$. Easy to see $h\geq 1$.
    By simple calculation, we have
    \begin{align*}
        \Pr_1\left(T_i \geq  \bar{T}_i\right)
        = & \Pr_1\left(T_i \geq  \bar{T}_i, \sum_{s=1}^{T_i}D_{i, \ell_0, s} < C_{\ell_0}, \sum_{s=1}^{\bar{T}_i}D_{i,  \ell_0, s} < C_{\ell_0}\right)\\
        \leq & \Pr_1\left(\sum_{s=1}^{\bar{T}_i}D_{i, \ell_0, s} < C_1\right)\\
        = & 1-\Pr_1\left(\sum_{s=1}^{\bar{T}_i}D_{i, \ell_0, s} \geq C_1\right)\\
        \leq & 1-\Pr_1\left(\sum_{s=1}^{\frac{C_{\ell_0}}{h}\log\frac{1}{d_{\ell_0, (i)}}}D_{i, \ell_0, s} \geq C_{\ell_0}\right).
    \end{align*}
    where $\{D_{i, \ell_0, s}\}_{s=1}^{+\infty}\stackrel{i.i.d}{\sim} Bern(d_{\ell_0,(i)})$.
    The last inequality is from the fact that $g(d_{\ell_0, k})<\frac{1}{\log\frac{1}{d_{\ell_0, (i)}}}$ holds for all $k$, further $\bar{T}_i > \frac{C_{\ell_0}}{h}\log\frac{1}{d_{\ell_0, (i)}}$. 
    
    %Denote $h = \left(\frac{1}{(r_1-r_2)^2} + \sum_{k=3}^K \frac{1}{(r_1-r_k)^2}\right) (r_1-r_i)^2$. Easy to see $h\geq 1$. 
    As $\frac{h}{\log\frac{1}{d_{\ell_0, (i)}}} > \frac{1}{\log \frac{1}{d_{\ell_0, (i)}}} > d_{\ell_0, (i)}$, we can apply lemma \ref{lemma:KL-Divergence-Chernoff-ln}.
    \begin{align}
        & \Pr_1\left(\sum_{s=1}^{\frac{C_{\ell_0}}{h}\log\frac{1}{d_{\ell_0, (i)}}}D_{i, \ell_0,  s} \geq C_1\right)\\
        \geq & \frac{\exp\left(-\frac{C_{\ell_0}}{h}\log\frac{1}{d_{\ell_0, (i)}}\text{KL}\left(\frac{h}{\log\frac{1}{d_{\ell_0, (i)}}}, d_{\ell_0,(i)}\right)\right)}{\frac{C_{\ell_0}}{h}\log\frac{1}{d_{\ell_0, (i)}} +1}\label{eqn:Apply-Lemma-KL}\\
        \geq & \frac{\exp\left(-(\frac{C_{\ell_0}}{h}\log\frac{1}{d_{\ell_0, (i)}})h -\frac{C_{\ell_0}}{h}\log\frac{1}{d_{\ell_0, (i)}}\log\frac{1}{1-d_{\ell_0,(i)}}\right)}{\frac{C_{\ell_0}}{h}\log\frac{1}{d_{\ell_0, (i)}} +1}\label{eqn:KL-Inequality}\\
        \geq & \exp\left(-\frac{C_{\ell_0}}{\frac{1}{\log\frac{1}{d_{\ell_0, (i)}}}} -\frac{C_{\ell_0}}{\frac{1}{\log\frac{1}{d_{\ell_0, (i)}}}}\frac{1}{2} - \frac{C_{\ell_0}}{\frac{1}{\log\frac{1}{d_{\ell_0, (i)}}}}\right)\label{eqn:h>=1-and-e^x>=x+1}\\
        \geq & \exp\left(-\frac{C_{\ell_0}}{\frac{1}{4\log\frac{1}{d_{\ell_0, (i)}}}}\right).
    \end{align}
    Step (\ref{eqn:Apply-Lemma-KL}) is by lemma \ref{lemma:KL-Divergence-Chernoff-ln}. Step (\ref{eqn:KL-Inequality}) is by the following fact
    \begin{align*}
        \text{KL}(\frac{M}{\log\frac{1}{d}}, d) = & \frac{M}{\log\frac{1}{d}} \log\frac{\frac{M}{\log\frac{1}{d}}}{d} + \left(1-\frac{M}{\log\frac{1}{d}}\right)\log\frac{1-\frac{M}{\log\frac{1}{d}}}{1-d}\\
        = & \frac{M}{\log\frac{1}{d}}\log\frac{M}{\log\frac{1}{d}} + M + \left(1-\frac{M}{\log\frac{1}{d}}\right) \log(1-\frac{M}{\log\frac{1}{d}}) + \left(1-\frac{M}{\log\frac{1}{d}}\right) \log\frac{1}{1-d}\\
        \leq & 0 + M + 0 + \left(1-\frac{M}{\log\frac{1}{d}}\right) \log\frac{1}{1-d}\\
        = & M+\log\frac{1}{1-d}
    \end{align*}
    holds for any $M, d$ such that $\frac{M}{\log\frac{1}{d}}, d\in (0, 1)$. Step (\ref{eqn:h>=1-and-e^x>=x+1}) is due to $h\geq 1$, $\log \frac{1}{1-d_{\ell_0, (1)}} < \frac{1}{2}$ and inequality $e^x\geq x+1$.
\end{proof}

Now we are ready to prove theorem \ref{theorem:lower-bound-sto-consumption-multiple-resource}. We firstly introduced some notations. Define $\tilde{H}^{\text{sto}}_{1, \ell}: = \frac{g(d_{\ell,(1)})}{(r_1-r_2)^2} + \sum_{k=3}^K \frac{g(d_{\ell,(k)})}{(r_1-r_k)^2}$, $H^{\text{sto}}_{1, \ell} = \frac{g(d_{\ell,(1)})}{(r_1-r_2)^2} + \sum_{k=2}^K \frac{g(d_{\ell,(k)})}{(r_1-r_k)^2}$. Easy to see 
% $2\tilde{H}^{\text{sto}}_{1, \ell}\geq H^{\text{sto}}_{\ell, 1}\geq \tilde{H}^{\text{sto}}_{1, \ell} > \tilde{H}$.
$2\tilde{H}^{\text{sto}}_{1, \ell}\geq H^{\text{sto}}_{1, \ell} > \tilde{H}_{2, \ell}^{\text{sto}}$. 
This implies once we prove
\begin{align*}
    \max_{j\in \{1, i\}}\Pr_{Q^{(j)}, alg}(failure) \geq \exp\left(-\min_{\ell\in [L]}\frac{C_{\ell}}{\tilde{H}^{\text{sto}}_{1, \ell}}\right),
\end{align*}
for small enough $\bar{c}$ and large enough $\{C_{\ell}\}_{\ell=1}^L$, we prove theorem \ref{theorem:lower-bound-sto-consumption-multiple-resource}. The constructions of $\bar{c}$ and $\{C_{\ell}\}_{\ell=1}^L$ are as follows. As $\lim\limits_{d\rightarrow 0^+}\frac{1}{g(d)\log\frac{1}{d}}=+\infty$ is equivalent to $\lim_{d\rightarrow 0^+} g(d)\log\frac{1}{d} = 0$, we can further conclude $\lim_{c\rightarrow 0^+} g(cd^0_{\ell, (i)})\log\frac{1}{cd^0_{\ell, (j)}} = \lim_{c\rightarrow 0^+} g(cd^0_{\ell, (i)})\log\frac{1}{cd^0_{\ell, (i)}} +g(cd^0_{\ell, (i)})\log\frac{d_{\ell,(i)}^0}{d_{\ell,(j)}^0} = 0, \forall i,j\in [K], \forall \ell\in [L]$, thus we can find a $\bar{c}$, such that when $0 < c<\bar{c}$,
\begin{itemize}
    \item $16g(d_{\ell, (j)}) = 16g(cd^0_{\ell, (j)}) < \frac{1}{\log\frac{1}{cd^0_{\ell, (i)}}} = \frac{1}{\log\frac{1}{d_{\ell, (i)}}}$ for all $j\in [K], \ell \in [L]$,
    \item $\log \frac{1}{1-d_{\ell, (j)}} < \frac{1}{2}$, $\log\frac{1}{d_{\ell, (j)}} > 1$ for all $j\in[K]$, $\ell \in [L]$,
    \item $64\left(\frac{g(d_{\ell, (1)})}{(r_1-r_2)^2} + \sum_{k=3}^K \frac{g(d_{\ell, (k)})}{(r_1-r_k)^2}\right)\log\frac{1}{d_{\ell, (i)}} < 2$, for all $\ell \in [L]$,
    \item $\frac{(r_1-r_i)^2}{(r_1-r_2)^2 \log\frac{1}{d_{\ell, (i)}}} + \sum_{k=3}^K \frac{(r_1-r_i)^2}{(r_1-r_k)^2\log\frac{1}{d_{\ell,(i)}}} < 1$, for all $\ell \in [L]$.
\end{itemize}
We also require $\{C_{\ell}\}_{\ell=1}^L$ are large enough, such that for the above given $\{d_{\ell, (k)}\}_{k=1, \ell=1}^{K,L}$, we have
\begin{itemize}
    \item $\frac{C_\ell}{16\tilde{H}^{\text{sto}}_{1, \ell}} > \sqrt{\frac{4}{16}\frac{C_1}{\tilde{H}^{\text{sto}}_{1, \ell}} \log \frac{1}{16}\frac{C_1}{\tilde{H}^{\text{sto}}_{1, \ell}} \frac{1}{(r_1-r_i)^2}}$ holds for all $\ell \in [L]$,
    \item $C_\ell\geq \log 64$ for all $\ell \in [L]$.
\end{itemize}

% {\color{red}ZITIAN: We might need to incorporate a $\frac{1}{16}$ in the definition of $\bar{T}_i$.}

For these $c$, $\{d_{\ell, (k)}\}_{k=1, \ell=1}^{K,L}$, and $\{C_\ell\}$, we can start the analysis. Define $\bar{T}_i = \min_{\ell\in[L]}\frac{1}{16}\frac{C_\ell}{\tilde{H}^{\text{sto}}_{1, \ell}} \frac{1}{(r_1-r_i)^2}$, $\hat{KL}_{i, s} = \log\frac{f_i(R_{i, s})}{f_{i'}(R_{i, s})}$, where $f_i$ is the density function of $\mathcal{N}(r_i, 1)$, $f_{i'}$ is the density function of $\mathcal{N}(1-r_i, 1)$, and $R_{i, s}\sim \mathcal{N}(r_i, 1)$. Easy to see
\begin{align*}
    \log\frac{f_i(R_{i, s})}{f_{i'}(R_{i, s})}
    = & \log e^{-\frac{(R_{i, s}-r_i)^2 - (R_{i, s}-(1-r_i))^2}{2}}\\
    = & -\frac{(R_{i, s}-r_i- R_{i, s}+(1-r_i))(R_{i, s}-r_i+R_{i, s}-(1-r_i))}{2}\\
    = & -\frac{2(\frac{1}{2}-r_i)(2R_{i, s}-1)}{2}\\
    = & -\frac{2(r_1-r_i)(2R_{i, s}-1)}{2} \sim \mathcal{N}(2(r_i-r_1)^2, 4(r_1-r_i)^2).
\end{align*}
Define $\xi_i=\left\{t \in [\bar{T}_i], \widehat{\text{KL}}_{i,t}-2(r_i-r_1)^2\le 
2|r_1-r_i|\cdot \sqrt{\frac{\min_{\ell \in [L]}\frac{C_{\ell}}{\tilde{H}^{\text{sto}}_{1, \ell}}+\log \bar{T}_i}{t}}
%\textsf{rad}(t)
\right\}$, easy to derive the following inequality by Chernoff and union bounds.
\begin{align*}
\Pr_1(\neg \xi_i)
\leq \sum_{t=1}^{\bar{T}_i} \exp\left(-\frac{4(r_1-r_a)^2\left(\frac{\min_{\ell \in [L]}\frac{C_{\ell}}{\tilde{H}^{\text{sto}}_{1, \ell}}+\log \bar{T}_i}{t}\right)}{4(r_1-r_i)^2} t\right)
= \exp\left(-\min_{\ell \in [L]}\frac{C_{\ell}}{\tilde{H}^{\text{sto}}_{1, \ell}}\right).
\end{align*}
Apply the transportaion equality just like section \ref{pf:thm_low_det}, we have
\begin{align}
    & \Pr_i(\psi \neq i)\nonumber\\
    \geq & \mathbb{E}_{1}\left(\mathds{1}\{\psi \neq i\}\mathds{1}\{T_i \leq \bar{T}_i\}\mathds{1}(\xi_i)\exp(-T_i\widehat{\text{KL}}_{i,T_i})\right)\\
    \geq & \mathbb{E}_{1}\left(\mathds{1}\{\psi \neq i\}\mathds{1}\{T_i \leq \bar{T}_i\}\mathds{1}(\xi_i)\exp\left( -\bar{T}_iKL_i - 2(r_1-r_i)\sqrt{\bar{T}_i\left(\min_{\ell \in [L]}\frac{C_{\ell}}{\tilde{H}^{\text{sto}}_{1, \ell}}+\log \bar{T}_i\right)}\right)\right)\\
    \stackrel{}{\geq} & \mathbb{E}_{1}\left(\mathds{1}\{\psi \neq i\}\mathds{1}\{T_i \leq \bar{T}_i\}\mathds{1}(\xi_i)\exp\left(-\bar{T}_iKL_i - \sqrt{4(r_1-r_i)^2\bar{T}_i\min_{\ell \in [L]}\frac{C_{\ell}}{\tilde{H}^{\text{sto}}_{1, \ell}}} - \sqrt{4(r_1-r_i)^2\bar{T}_i\log \bar{T}_i}\right)\right)\label{eqn:Inequality-Square-Root}\\
    = & \mathbb{E}_{1}\Bigg(\mathds{1}\{\psi \neq i\}\mathds{1}\{T_i \leq \bar{T}_i\}\mathds{1}(\xi_i)\nonumber\\
    &\exp\left(-\frac{2}{16}\min_{\ell \in [L]}\frac{C_\ell}{\tilde{H}^{\text{sto}}_{1, \ell}} - \frac{8}{16}\min_{\ell \in [L]}\frac{C_\ell}{\tilde{H}^{\text{sto}}_{1, \ell}} - \sqrt{\frac{4}{16}\min_{\ell \in [L]}\frac{C_\ell}{\tilde{H}^{\text{sto}}_{1, \ell}} \log \frac{1}{16}\min_{\ell \in [L]}\frac{C_\ell}{\tilde{H}^{\text{sto}}_{1, \ell}} \frac{1}{(r_1-r_i)^2}}\right)\Bigg)\\
    \stackrel{}{\geq} & \Pr_{1}\left((\psi \neq i) \text{ and } (T_i \leq \bar{T}_i) \text{ and }\xi_i\right)\exp\left(-\frac{11}{16}\min_{\ell \in [L]}\frac{C_\ell}{\tilde{H}^{\text{sto}}_{1, \ell}}\right).\label{eqn:C_ell-Large-Enough}
    % = & \Pr_{1}\left((\psi \neq i) \text{ and } (T_i \leq \bar{T}_i) \text{ and }\xi_i\right)\exp\left(-\frac{11}{16}\min_{\ell \in [L]}\frac{C_\ell}{\tilde{H}^{\text{sto}}_{1, \ell}}\right).
\end{align}
Step (\ref{eqn:Inequality-Square-Root}) is by the inequality $\sqrt{a+b} \leq \sqrt{a}+\sqrt{b}$ for $a, b\geq 0$. Step (\ref{eqn:C_ell-Large-Enough}) is by the requirement $\frac{C_\ell}{16\tilde{H}^{\text{sto}}_{1, \ell}} > \sqrt{\frac{4}{16}\frac{C_1}{\tilde{H}^{\text{sto}}_{1, \ell}} \log \frac{1}{16}\frac{C_1}{\tilde{H}^{\text{sto}}_{1, \ell}} \frac{1}{(r_1-r_i)^2}}$ holds for all $\ell \in [L]$.

We can further derive the lower bound of the probabilistic term,
\begin{align*}
    &\Pr_{1}\left((\psi \neq i) \text{ and } (T_i \leq \bar{T}_i) \text{ and }\xi_i\right)\\
    \geq&1- \Pr_1\left(\psi = i\right) - \Pr_1\left(T_i >  \bar{T}_i\right)-\Pr_1\left( \neg\xi_i\right)\\ 
    \stackrel{}{\geq}&1-\exp\left(-\min_{\ell \in [L]}\frac{C_\ell}{\tilde{H}^{\text{sto}}_{1, \ell}}\right)-\Pr_1\left(T_i > \bar{T}_i\right)-\exp\left(-\min_{\ell \in [L]}\frac{C_\ell}{\tilde{H}^{\text{sto}}_{1, \ell}}\right).
\end{align*}
The last inequality is by the assumption $\Pr_1\left(\psi = i\right)\leq \exp\left(-\min_{\ell \in [L]}\frac{C_\ell}{\tilde{H}^{\text{sto}}_{1, \ell}}\right)$. If this assumption doesn't hold, then we have completed the proof. Denote $\ell_0 = \arg\min_{\ell\in [L]}\frac{C_\ell}{\frac{g(d_{\ell, (1)})}{(r_1-r_2)^2}+\sum_{k=3}^K \frac{g(\ell, d_{(k)})}{(r_1-r_k)^2}}$, apply lemma \ref{lemma:Prob-Bound-of-Stopping-Time} to $\Pr_1\left(T_i > \bar{T}_i\right)$, we get
\begin{align*}
    & \Pr_{1}\left((\psi \neq i) \text{ and } (T_i \leq \bar{T}_i) \text{ and }\xi_i\right)\\
    \geq & 1-\exp\left(-\frac{C_{\ell_0}}{\tilde{H}^{\text{sto}}_{1, \ell_0}}\right)-\left(1-\exp\left(-\frac{C_{\ell_0}}{\frac{1}{4\log\frac{1}{d_{\ell_0, (i)}}}}\right)\right)-\exp\left(-\frac{C_{\ell_0}}{\tilde{H}^{\text{sto}}_{1, \ell_0}}\right)\\
    = & \exp\left(-\frac{C_{\ell_0}}{\frac{1}{4\log\frac{1}{d_{\ell_0, (i)}}}}\right) - \exp\left(-\frac{C_{\ell_0}}{\tilde{H}^{\text{sto}}_{1, \ell_0}}\right) - \exp\left(-\frac{C_{\ell_0}}{\tilde{H}^{\text{sto}}_{1, \ell_0}}\right).
\end{align*}
By the property of $c, d_{\ell_0, (i)}$, easy to see
\begin{align*}
    & 64\left(\frac{g(d_{\ell_0, (1)})}{(r_1-r_2)^2} + \sum_{k=3}^K \frac{g(d_{\ell_0, (k)})}{(r_1-r_k)^2}\right)\log\frac{1}{d_{\ell_0, (i)}} < 2, C_{\ell_0} > \log 64, \log\frac{1}{d_{\ell_0, (i)}} > 1\\
    \Rightarrow & \exp\left(-\frac{C_{\ell_0}}{\tilde{H}^{\text{sto}}_{1, \ell_0}}\right) \leq \frac{1}{4}\exp\left(-\frac{C_{\ell_0}}{\frac{1}{4\log\frac{1}{d_{\ell_0, (i)}}}}\right).
\end{align*}
Thus, we can conclude $\Pr_{1}\left((\psi \neq i) \text{ and } (T_i \leq \bar{T}_i) \text{ and }\xi_i\right)\geq \frac{1}{2}\exp\left(-\frac{C_{\ell_0}}{\frac{1}{4\log\frac{1}{d_{\ell_0, (i)}}}}\right)$, further 
\begin{align*}
    \Pr_i(\psi \neq i) \geq \frac{1}{2}\exp\left(-\frac{C_{\ell_0}}{\frac{1}{4\log\frac{1}{d_{\ell_0, (i)}}}}\right)\exp\left(-\frac{11}{16}\frac{C_{\ell_0}}{\tilde{H}^{\text{sto}}_{1, \ell_0}}\right).
\end{align*}
That implies
\begin{align*}
    & \frac{\Pr_i(\psi \neq i)}{\exp\left(-\frac{C_{\ell_0}}{\tilde{H}^{\text{sto}}_{1, \ell_0}}\right) }\\
    \geq & \exp\left(\frac{5C_{\ell_0}}{16\tilde{H}^{\text{sto}}_{1, \ell_0}} - \frac{C_{\ell_0}}{\frac{1}{4\log\frac{1}{d_{\ell_0,(i)}}}}\right)\\
    = & \exp\left(\frac{\left(5 - 64\left(\frac{g(d_{\ell_0,(1)})}{(r_1-r_2)^2} + \sum_{k=3}^K \frac{g(d_{\ell_0,(k)})}{(r_1-r_k)^2}\right)\log\frac{1}{d_{\ell_0,(i)}} \right)C_{\ell_0,}}{16\tilde{H}^{\text{sto}}_{1, \ell_0}}\right)\\
    \geq & \exp\left(\frac{C_{\ell_0}}{16\tilde{H}^{\text{sto}}_{1, \ell_0}}\right) > 1.
\end{align*}
The last second inequality is from the property $64\left(\frac{g(d_{\ell, (1)})}{(r_1-r_2)^2} + \sum_{k=3}^K \frac{g(d_{\ell, (k)})}{(r_1-r_k)^2}\right)\log\frac{1}{d_{\ell, (i)}} < 2$, for all $\ell \in [L]$. The overall inequality suggests
\begin{align*}
    \Pr_i(\psi \neq i) \geq \exp\left(-\frac{C_{\ell_0}}{\tilde{H}^{\text{sto}}_{1, \ell_0}}\right) = \exp\left(-\min_{\ell\in [L]}\frac{C_{\ell}}{\tilde{H}^{\text{sto}}_{1, \ell}}\right) \geq \exp\left(-\min_{\ell\in [L]}\frac{C_{\ell}}{\frac{1}{2}\left(\frac{g(d_{\ell, (1)})}{(r_1-r_2)^2}+\sum_{k=2}^K \frac{g(d_{\ell, (k)})}{(r_1-r_k)^2}\right)}\right),
\end{align*}
which is our target.

\subsection{Improvement of $H_{2, \ell}^{\text{det}}(Q)$ is Unachievable}
\label{sec:Improvement_of_H2_is_unachievable}
For a problem instance $Q$ with mean reward $\{r^Q_k\}_{k=1}^K, r^Q_1\geq r^Q_2\geq\cdots \geq r^Q_K$, mean consumption $\{d^Q_{\ell, k}\}_{k=1, \ell=1}^{K, L}$ and budget $\{C_{\ell}\}_{\ell=1}^L$, we define 
\begin{align}
    \tilde{H}_{1, \ell}^{det}(Q)=&\frac{d_{\ell, 1}}{(r^Q_1-r^Q_2)^2}+\sum_{k=2}^{K}\frac{d_{\ell, k}}{(r^Q_1-r^Q_k)^2}\label{eq:refined_H1_det}\\
    \tilde{H}_{2, \ell}^{det}(Q)=&\max_{2\leq k \leq K} \frac{\sum_{j=1}^k d_{\ell, j}}{(r^Q_1-r^Q_k)^2}\label{eq:refined_H2_det}.
\end{align}
Easy to see $\tilde{H}_{1, \ell}^{det}(Q)\leq H_{1,\ell}^{det}(Q)$, $\tilde{H}_{2, \ell}^{det}(Q)\leq H_{2,\ell}^{det}(Q)$. We want to know whether we can find an algorithm such that for any problem instance $Q$, we can achieve the following upper bound of the failure probability. 
\begin{align}
    & \Pr_Q(failure) \nonumber \\
    \le & poly(K) \exp\left(-\frac{O(1)}{\log_2 K}\min_{\ell\in [L]}\left\{\frac{C_{\ell}}{\tilde{H}_{2, \ell}^{det}(Q)}\right\}\right)\label{eq:refined_upper_det}.
\end{align}
\textbf{The answer is No.} And the analysis method we used is similar to appendix \ref{pf:thm_low_det}. We can construct a list of problem instance $Q^{(i)}$, and prove a lower bound that could be larger than the right side of (\ref{eq:refined_upper_det}), as $K$ and $\{C_{\ell}\}_{\ell=1}^L$ are large enough.

We focus on $L=1$. Assume there are $C$ units of the resource. Given $K$, let $d_1=\frac{1}{2^{K-2}}, d_k = \frac{1}{2^{K-k}}, k\ge 2$, $r_1=\frac{1}{2}, r_k=\frac{1}{2}-2^{\frac{k-K-4}{2}},k\geq2$. Easy to see $d_1= d_2\leq \cdots \leq d_K$, $\frac{1}{2}=r_1\geq r_2\geq \cdots \geq r_K=\frac{1}{4}$. Then we construct $K$ problem instances $\{Q^{(i)}\}_{i=1}^K$. For problem instance $Q^{(1)}$, the mean reward of $k^{th}$ arm is $r_k$, following the Bernoulli distribution. And the deterministic consumption of $k^{th}$ arm is $d_k$. For problem instance $Q^{(i)}, 2\leq i\leq K$, the mean reward of $k^{th}\ne i$ arm is $r_k$, the mean reward of $i^{th}$ arm is $1-r_i$. And the deterministic consumption of $k^{th}$ arm is $d_k$. For $i\in [K]$, the best arm of $Q^{(i)}$ is always the $i^{th}$ arm. Then we can calculate $\tilde{H}_{1,\ell=1}^{det}(Q^{(i)})$ and $\tilde{H}_{2,\ell=1}^{det}(Q^{(i)})$ for $i\in [K]$. Easy to derive
\begin{align*}
    &\tilde{H}_{1,\ell=1}^{det}(Q^{(1)})\\
    =&\frac{d_1}{(r_1-r_2)^2}+\sum_{k=2}^K \frac{d_k}{(r_1-r_k)^2}\\
    =&\frac{\frac{K}{2}}{2^{K-3}\frac{1}{2^{K+2}}}=16K.
\end{align*}
\begin{align}
    \tilde{H}_{2, \ell=1}^{det}(Q^{(1)})=\max_{k\ge 2} \frac{\sum_{t=1}^k d_{t}}{(r_1-r_k)^2}=32.
\end{align}
For $2\leq i\leq K$, 
\begin{align}
    &\tilde{H}_{1,\ell=1}^{det}(Q^{(i)})\nonumber\\
    =&\frac{d_i+d_1}{(1-r_i-r_1)^2}+\sum_{t=2}^{i-1} \frac{d_t}{(1-r_i-r_t)^2}+\nonumber\\
    &\sum_{t=i+1}^{K} \frac{d_t}{(1-r_i-r_t)^2}\nonumber\\
    =&\frac{\frac{1}{2^{K-i}}+\frac{1}{2^{K-2}}}{(2^{\frac{i-K-4}{2}})^2}+\sum_{t=2}^{i-1}\frac{\frac{1}{2^{K-t}}}{(2^{\frac{i-K-4}{2}}+2^{\frac{t-K-4}{2}})^2}+\nonumber\\
    &\sum_{t=i+1}^{K}\frac{\frac{1}{2^{K-t}}}{(2^{\frac{i-K-4}{2}}+2^{\frac{t-K-4}{2}})^2}\nonumber\\
    =&\frac{2^i+4}{2^{i-4}}+\sum_{t=2}^{i-1}\frac{2^t}{2^{i-4}+2^{\frac{i+t}{2}-3}+2^{t-4}}+\nonumber\\
    &\sum_{t=i+1}^{K}\frac{2^t}{2^{i-4}+2^{\frac{i+t}{2}-3}+2^{t-4}}.
\end{align}
\begin{align}
    &\tilde{H}_{2, \ell=1}^{det}(Q^{(i)})\nonumber\\
    =&\max\{\frac{d_i+d_1}{(1-r_i-r_1)^2}, \max_{2\leq t \leq i-1} \frac{d_i+\sum_{l=1}^{t}d_l}{(1-r_i-r_t)^2},\max_{i+1\leq t\leq K} \frac{\sum_{l=1}^{t}d_l}{(1-r_i-r_t)^2}\}\nonumber\\
    =&\max\{\frac{\frac{1}{2^{K-i}}+\frac{1}{2^{K-2}}}{(2^{\frac{i-K-4}{2}})^2}, \max_{2\leq t \leq i-1}\frac{\frac{1}{2^{K-i}}+\frac{1}{2^{K-t-1}}}{(2^{\frac{i-K-4}{2}}+2^{\frac{t-K-4}{2}})^2},\max_{i+1\leq t\leq K}\frac{\frac{1}{2^{K-t-1}}}{(2^{\frac{i-K-4}{2}}+2^{\frac{t-K-4}{2}})^2}\}\nonumber\\
    =&\max\{\frac{2^i+4}{2^{i-4}}, \max_{2\leq t \leq i-1}\frac{2^i+2^{t+1}}{2^{i-4}+2^{\frac{i+t}{2}-3}+2^{t-4}}, \max_{i+1\leq t\leq K}\frac{2^{t+1}}{2^{i-4}+2^{\frac{i+t}{2}-3}+2^{t-4}}\}.
\end{align}
With a simple calculation, for $2\leq i\leq K$, we have $\frac{2^i+4}{2^{i-4}}\leq 32$, $\frac{2^i+2^{t+1}}{2^{i-4}+2^{\frac{i+t}{2}-3}+2^{t-4}}\leq 32$ and $\frac{2^{t+1}}{2^{i-4}+2^{\frac{i+t}{2}-3}+2^{t-4}}\leq 32$. Thus we can conclude $\tilde{H}_{2, \ell=1}^{det}(Q^{(i)})\leq32=\tilde{H}_{2, \ell=1}^{det}(Q^{(1)})$. On the other hand, easy to check $\tilde{H}_{1,\ell=1}^{det}(Q^{(i)})\leq \tilde{H}_{1,\ell=1}^{det}(Q^{(1)})$ from the definition of $\tilde{H}_{1,\ell=1}^{det}$.

Following the step 1 in appendix \ref{pf:thm_low_det}, we can conclude
for every $i\in \{2, \ldots, K\}$ it holds that
% \begin{equation}%\label{eq:det_first}
%    \Pr_i(\psi\neq i) \geq \frac{1}{6}\exp\left(-60t_i\left(\frac{1}{2}-r_i\right)^2 -2\sqrt{T\log(12KT)}\right),
% \end{equation}
\begin{align}%\label{eq:det_first}
    \begin{split}
        & \Pr_i(\psi\neq i)\\
        \geq & \frac{1}{6}\exp\left(-60t_i\left(\frac{1}{2}-r_i\right)^2 -2\sqrt{T\log(12KT)}\right),
    \end{split}
\end{align}
where 
\begin{equation}%\label{eq:tiTi_det}
t_i = \mathbb{E}_1[T_i], \quad T_i = \sum^\tau_{t=1} \mathbf{1}(A(t) = i)
\end{equation}
is the number of times pulling arm $i$, and 
\begin{equation}%\label{eq:T_det}
T = \lfloor\frac{C}{d_{1}}\rfloor
\end{equation}
is an upper bound to the number of arm pulls by any policy that satisfies the resource constraints with certainty.

Following the step 2 in appendix \ref{pf:thm_low_det}, recall $d_1=d_2=\frac{1}{2^{K-2}}$, we can derive
\begin{align*}
    \sum_{k=2}^K \frac{2d_k}{\tilde{H}_{1,\ell=1}^{det}(Q^{(1)}) (r_1-r_k)^2}\ge 1.
\end{align*}
Since $\sum_{k=1}^K t_k d_{k}\le C$, we can further conclude
\begin{align*}
    \sum_{k=2}^K \frac{2C d_k}{\tilde{H}_{1,\ell=1}^{det}(Q^{(1)}) (r_1-r_k)^2}\ge \sum_{k=1}^K t_k d_{k}
\end{align*}
which implies there exists $i\ge2$, such that $\frac{2C d_i}{\tilde{H}_{1,\ell=1}^{det}(Q^{(1)}) (r_1-r_i)^2}\ge t_i d_{i}$. For this $i$,
\begin{align*}
    &\mathbb{P}_{\mathcal{G}^i}(\hat{k}\ne i)\\
    \ge&\frac{1}{6}\exp\Bigg(-120\frac{C}{\tilde{H}_{1,\ell=1}^{det}(Q^{(1)})} -\\
    &\sqrt{2}\log3\sqrt{\lfloor\frac{C}{d_{(k)}}\rfloor\log(12\lfloor\frac{C}{d_{(k)}}\rfloor K)}\Bigg).
\end{align*}
When $C$ is large enough, we can assume $120\frac{C}{H_1(Q^{(1)})}>\sqrt{2}\log3\sqrt{\lfloor\frac{C}{d_{(k)}}\rfloor\log(12\lfloor\frac{C}{d_{1}}\rfloor K)}$, for any bandit strategy that returns the arm $\hat{k}$,
\begin{align*}
    \max_{2\le i\le K} \mathbb{P}_{\mathcal{G}^i}(\hat{k}\ne i)\ge &\frac{1}{6}\exp\left(-240\frac{C}{\tilde{H}_{1,\ell=1}^{det}(Q^{(1)})} \right)\\
    =&\frac{1}{6}\exp\left(-480\frac{C}{K \tilde{H}_{2,\ell=1}^{det}(Q^{(1)})} \right)\\
    \ge& \frac{1}{6}\exp\left(-480\frac{C}{K \tilde{H}_{2,\ell=1}^{det}(Q^{(i)})} \right).
\end{align*}
That means we should \textbf{never} expect to use the right side of \ref{eq:refined_upper_det} as a general upper bound.
 
\section{Details on the Numerical Experiment Set-ups}
\label{sec:Details-on-the-Numerical-Experiment-Set-ups}
In what follows, we provide details about how the numerical experiments are run. All the numerical experiments were run on the Kaggle servers. 
\subsection{Single Resource, i.e, $L=1$}
\label{sec:Details-on-the-Numerical-Experiment-Set-ups-single}
The details about Figure \ref{fig:numeric-result} is as follows. Firstly, the bars in the plot are more detailedly explained as follows: From left to right, the 1st blue column is matching high reward and high consumption, considering deterministic resource consumption. The 2nd orange column is matching high reward and low consumption, considering deterministic consumption. The 3rd green column is matching high reward and high consumption, considering correlated reward and consumption. The 4th red column is matching high reward and low consumption, considering correlated reward and consumption. The 5th purple column is matching high reward and high consumption, considering uncorrelated reward and consumption. The 6th brown column is matching high reward and low consumption, considering uncorrelated reward and consumption.

Next, we list down the detailed about the setup.
\begin{enumerate}
    \item One group of suboptimal arms, High match High
    
    $r_1=0.9$; $r_i=0.8, i=2,\cdots,256$; $d_{\ell=1, i}=0.9, i=1,\cdots,128$; $d_{\ell=1, i}=0.1, i=129,\cdots,256$
    
    \item One group of suboptimal arms, High match Low
    
    $r_1=0.9$; $r_i=0.8, i=2,\cdots,256$; $d_{\ell=1, i}=0.1, i=1,\cdots,128$; $d_{\ell=1, i}=0.9, i=129,\cdots,256$
    
    \item Trap, High match High
    
    $r_1=0.9$; $r_i=0.8, i=2,\cdots,32$; $r_i=0.1, i=33,\cdots,256$; $d_{\ell=1, i}=0.9, i=1,\cdots,128$; $d_{\ell=1, i}=0.1, i=129,\cdots,256$
    
    \item Trap, High match Low
    
    $r_1=0.9$; $r_i=0.8, i=2,\cdots,32$; $r_i=0.1, i=33,\cdots,256$; $d_{\ell=1, i}=0.1, i=1,\cdots,128$; $d_{\ell=1, i}=0.9, i=129,\cdots,256$
    
    \item Polynomial, High match High
    
    $r_1=0.9, r_i=0.9(1-\sqrt{\frac{i}{256}}), i\geq 2$. $d_{\ell=1, i}=0.9, i=1,\cdots,128$; $d_{\ell=1, i}=0.1, i=129,\cdots,256$
    
    \item Polynomial, High match Low
    
    $r_1=0.9, r_i=0.9(1-\sqrt{\frac{i}{256}}), i\geq 2$. $d_{\ell=1, i}=0.1, i=1,\cdots,128$; $d_{\ell=1, i}=0.9, i=129,\cdots,256$
    
    \item Geometric, High match High
    
    $r_1=0.9, r_{256}=0.1$, $\{r_i\}_{i=1}^{256}$ is geometric, $r_i = 0.9 * (\frac{1}{9})^{\frac{i-1}{255}}$. $d_{\ell=1, i}=0.9, i=1,\cdots,128$; $d_{\ell=1, i}=0.1, i=129,\cdots,256$
    
    \item Geometric, High match Low
    
    $r_1=0.9, r_{256}=0.1$, $\{r_i\}_{i=1}^{256}$ is geometric, $r_i = 0.9 * (\frac{1}{9})^{\frac{i-1}{255}}$. $d_{\ell=1, i}=0.1, i=1,\cdots,128$; $d_{\ell=1, i}=0.9, i=129,\cdots,256$
\end{enumerate}
There are three kinds of consumption.
% \begin{enumerate}
%     \item Deterministic Consumption. The consumption of each arm are deterministic. The reward of each arm follows Gaussian Distribution with a standard deviation 0.5.
%     \item Uncorrelated Consumption. When we pull an arm, the consumption and reward are independent. The consumptions follow Bernoulli Distribution and the reward follow Gaussian Distribution with a standard deviation 0.5.
%     \item Correlated Consumption. When we pull the arm $i$, the consumption is $D_{\ell=1, i}=\mathds{1}(U\le d_{\ell=1,i})$, $R_i=r_i + sgn(U-0.5) * |N(0, 1/4)|$, where $U$ follows uniform distribution on $[0,1]$.
% \end{enumerate}
\begin{enumerate}
    \item Deterministic Consumption. The consumption of each arm are deterministic.
    \item Uncorrelated Consumption. When we pull an arm, the consumption and reward follow Bernoulli Distribution and are independent.
    \item Correlated Consumption. When we pull the arm $i$, the consumption is $D_{\ell=1, i}=\mathds{1}(U\le d_{\ell=1,i})$, $D_{\ell=2, i}=\mathds{1}(U\le d_{\ell=2,i})$, $R=\mathds{1}(U\le r_i)$, where $U$ follows uniform distribution on $[0,1]$
\end{enumerate}

\subsection{Multiple Resources}
\label{sec:Details-on-the-Numerical-Experiment-Set-ups-multiple}
Similarly, in multiple resources cases, we still considered different setups of mean reward, consumption, and consumption setups. 
\begin{enumerate}
    \item One group of suboptimal arms, High match High
    
    $r_1=0.9$; $r_i=0.8, i=2,\cdots,256$; $d_{\ell=1, i}=0.9, i=1,\cdots,128$; $d_{\ell=1, i}=0.1, i=129,\cdots,256$; $d_{\ell=2, i}=0.9, i=1,\cdots,128$; $d_{\ell=2, i}=0.1, i=129,\cdots,256$
    
    \item One group of suboptimal arms, Mixture
    
    $r_1=0.9$; $r_i=0.8, i=2,\cdots,256$; $d_{\ell=1, i}=0.1, i=1,\cdots,128$; $d_{\ell=1, i}=0.9, i=129,\cdots,256$; $d_{\ell=2, i}=0.9, i=1,\cdots,128$; $d_{\ell=2, i}=0.1, i=129,\cdots,256$
    
    \item One group of suboptimal arms, High match Low
    
    $r_1=0.9$; $r_i=0.8, i=2,\cdots,256$; $d_{\ell=1, i}=0.1, i=1,\cdots,128$; $d_{\ell=1, i}=0.9, i=129,\cdots,256$; $d_{\ell=2, i}=0.1, i=1,\cdots,128$; $d_{\ell=2, i}=0.9, i=129,\cdots,256$
    
    \item Trap, High match High
    
    $r_1=0.9$; $r_i=0.8, i=2,\cdots,32$; $r_i=0.1, i=33,\cdots,256$; $d_{\ell=1, i}=0.9, i=1,\cdots,128$; $d_{\ell=1, i}=0.1, i=129,\cdots,256$; $d_{\ell=2, i}=0.9, i=1,\cdots,128$; $d_{\ell=2, i}=0.1, i=129,\cdots,256$;
    
    \item Trap, Mixture
    
    $r_1=0.9$; $r_i=0.8, i=2,\cdots,32$; $r_i=0.1, i=33,\cdots,256$; $d_{\ell=1, i}=0.1, i=1,\cdots,128$; $d_{\ell=1, i}=0.9, i=129,\cdots,256$; $d_{\ell=2, i}=0.9, i=1,\cdots,128$; $d_{\ell=2, i}=0.1, i=129,\cdots,256$;
    
    \item Trap, High match Low
    
    $r_1=0.9$; $r_i=0.8, i=2,\cdots,32$; $r_i=0.1, i=33,\cdots,256$; $d_{\ell=1, i}=0.1, i=1,\cdots,128$; $d_{\ell=1, i}=0.9, i=129,\cdots,256$; $d_{\ell=2, i}=0.1, i=1,\cdots,128$; $d_{\ell=2, i}=0.9, i=129,\cdots,256$
    
    \item Polynomial, High match High
    
    $r_1=0.9, r_i=0.9(1-\sqrt{\frac{i}{256}}), i\geq 2$. $d_{\ell=1, i}=0.9, i=1,\cdots,128$; $d_{\ell=1, i}=0.1, i=129,\cdots,256$; $d_{\ell=2, i}=0.9, i=1,\cdots,128$; $d_{\ell=2, i}=0.1, i=129,\cdots,256$
    
    \item Polynomial, Mixture
    
    $r_1=0.9, r_i=0.9(1-\sqrt{\frac{i}{256}}), i\geq 2$. $d_{\ell=1, i}=0.1, i=1,\cdots,128$; $d_{\ell=1, i}=0.9, i=129,\cdots,256$; $d_{\ell=2, i}=0.9, i=1,\cdots,128$; $d_{\ell=2, i}=0.1, i=129,\cdots,256$
    
    \item Polynomial, High match Low
    
    $r_1=0.9, r_i=0.9(1-\sqrt{\frac{i}{256}}), i\geq 2$. $d_{\ell=1, i}=0.1, i=1,\cdots,128$; $d_{\ell=1, i}=0.9, i=129,\cdots,256$; $d_{\ell=2, i}=0.1, i=1,\cdots,128$; $d_{\ell=2, i}=0.9, i=129,\cdots,256$
    
    \item Geometric, High match High
    
    $r_1=0.9, r_{256}=0.1$, $\{r_i\}_{i=1}^{256}$ is geometric, $r_i = 0.9 * (\frac{1}{9})^{\frac{i-1}{255}}$. $d_{\ell=1, i}=0.9, i=1,\cdots,128$; $d_{\ell=1, i}=0.1, i=129,\cdots,256$; $d_{\ell=2, i}=0.9, i=1,\cdots,128$; $d_{\ell=2, i}=0.1, i=129,\cdots,256$;
    
    \item Geometric, Mixture
    
    $r_1=0.9, r_{256}=0.1$, $\{r_i\}_{i=1}^{256}$ is geometric, $r_i = 0.9 * (\frac{1}{9})^{\frac{i-1}{255}}$. $d_{\ell=1, i}=0.1, i=1,\cdots,128$; $d_{\ell=1, i}=0.9, i=129,\cdots,256$; $d_{\ell=2, i}=0.9, i=1,\cdots,128$; $d_{\ell=2, i}=0.1, i=129,\cdots,256$;
    
    \item Geometric, High match Low
    
    $r_1=0.9, r_{256}=0.1$, $\{r_i\}_{i=1}^{256}$ is geometric, $r_i = 0.9 * (\frac{1}{9})^{\frac{i-1}{255}}$. $d_{\ell=1, i}=0.1, i=1,\cdots,128$; $d_{\ell=1, i}=0.9, i=129,\cdots,256$; $d_{\ell=2, i}=0.1, i=1,\cdots,128$; $d_{\ell=2, i}=0.9, i=129,\cdots,256$
\end{enumerate}
There are two kinds of consumption.
% \begin{enumerate}
%     \item Uncorrelated Consumption. When we pull an arm, the consumption and reward follow Bernoulli Distribution and are independent.
%     \item Correlated Consumption. When we pull the arm $i$, the consumption is $D_{\ell=1, i}=\mathds{1}(U\le d_{\ell=1,i})$, $D_{\ell=2, i}=\mathds{1}(U\le d_{\ell=2,i})$, $R_i=r_i + sgn(U-0.5) * |N(0, 1/4)|$, where $U$ follows uniform distribution on $[0,1]$.
% \end{enumerate}
\begin{enumerate}
    \item Uncorrelated Consumption. When we pull an arm, the consumption and reward follow Bernoulli Distribution and are independent.
    \item Correlated Consumption. When we pull the arm $i$, the consumption is $D_{\ell=1, i}=\mathds{1}(U\le d_{\ell=1,i})$, $D_{\ell=2, i}=\mathds{1}(U\le d_{\ell=2,i})$, $R=\mathds{1}(U\le r_i)$, where $U$ follows uniform distribution on $[0,1]$
\end{enumerate}

\subsection{Detailed Setting of Real-World Dataset}
\label{sec:details-realworld}
We adopted K Nearest Neighbour, Logistic Regression, Random Forest, and Adaboost as our candidates for the classifiers. And we applied each combination of machine learning model and its hyper-parameter to each supervised learning task with 500 independent trials. We identified the combination with the lowest empirical mean cross-entropy as the best arm. 

Our BAI experiments were conducted across 100 independent trials. During each arm pull in a BAI experiment round—i.e., selecting a machine learning model with a specific hyperparameter combination—we partitioned the datasets randomly into training and testing subsets, maintaining a testing fraction of 0.3. The training subset was utilized to train the machine learning models, and the cross-entropy computed on the testing subset served as the realized reward. We flattened the 2-D image as a vector if the dataset is consists of images. All the experiments are deploied on the Kaggle Server with default CPU specifications. 

The details of the real-world datasets we used are as follows
\begin{itemize}
    \item To classify labels 3 and 8 in part of the MNIST Dataset. (MNIST 3\&8)
    
    Number of label 3: 1086, Number of label 8: 1017, Number of Atrributes: 784.

    Time budget: 60 seconds.

    Link of dataset: \url{https://www.kaggle.com/competitions/digit-recognizer}.
    \item Optical Recognition of Handwritten Digits Data Set. (Handwritten)

    Number of Instances: 3823, Number of Attributes: 64

    Time budget: 60 seconds.

    Link of dataset: \url{https://archive.ics.uci.edu/ml/datasets/optical+recognition+of+handwritten+digits}
    \item To classify labels -1 and 1 in the MADELON dataset. (MADELON)

    Number of Instances: 2000, Number of Attributes: 500.

    Time budget: 80 seconds.

    Link of dataset: \url{https://archive.ics.uci.edu/ml/datasets/Madelon}. 
    \item To classify labels -1 and 1 in the Arcene dataset. (Arcene)
    
    Number of Instances: 200, Number of Attributes:10000

    Time budget: 150 seconds.

    Link of dataset: \url{https://archive.ics.uci.edu/ml/datasets/Arcene} (Arcene)
    \item To classify labels of weight conditions in the Obesity dataset. (Obesity)

    Number of Instances: 2111, Number of Attributes: 16.

    Time budget: 20 seconds.
    
    Link of dataset:
    \url{https://archive.ics.uci.edu/dataset/544/estimation+of+obesity+levels+based+on+eating+habits+and+physical+condition}.
\end{itemize}

The machine learning models and candidate hyperparameters, aka arms, are as follows. We implemented all these models through the scikit-learn package in https://scikit-learn.org/stable/index.html
\begin{itemize}
    \item K Nearest Neighbour
    \begin{itemize}
        \item n\_neighbours = 5, 15, 25, 35, 45, 55, 65, 75
    \end{itemize}
    \item Logistic Regression
    \begin{itemize}
        \item Regularization = "l2" or None
        \item Intercept exists or not exists
        \item Inverse value of regularization coefficient= 1, 2
    \end{itemize}
    \item Random Forest
    \begin{itemize}
        \item Fix max\_depth = 5
        \item n\_estimators = 10, 20, 30, 50
        \item criterion = "gini" or "entropy"
    \end{itemize}
    \item Adaboost
    \begin{itemize}
        \item n\_estimators = 10, 20, 30, 40
        \item learning rate = 1.0, 0.1
    \end{itemize}
\end{itemize}

\end{document}